\documentclass{article}

\usepackage{arxiv}
\usepackage{algorithm,algorithmic}
\usepackage[utf8]{inputenc} % allow utf-8 input
\usepackage[T1]{fontenc}    % use 8-bit T1 fonts
\usepackage{hyperref}       % hyperlinks
\usepackage{url}            % simple URL typesetting
\usepackage{booktabs}       % professional-quality tables
\usepackage{amsfonts}       % blackboard math symbols
\usepackage{nicefrac}       % compact symbols for 1/2, etc.
\usepackage{microtype}      % microtypography
\usepackage[english]{babel}
\usepackage{mathrsfs}
\usepackage{appendix}
\usepackage{geometry}
\usepackage{times}
\usepackage{latexsym}
\usepackage[]{mathtools}
\usepackage[]{amsthm} % provides proof, theorem, lemma, etc. environments
\usepackage[]{amssymb}
\usepackage{bbm}
\usepackage{subfigure}

\usepackage{xcolor}
\newtheorem{theorem}{Theorem}

\newtheorem{lemma}{Lemma}

\newtheorem{definition}{Definition}

\usepackage[]{mathtools} 
\def\##1\#{\begin{align}#1\end{align}}
\def\$#1\${\begin{align*}#1\end{align*}}
\usepackage{enumitem}

\newcommand{\RNum}[1]{\uppercase\expandafter{\romannumeral #1\relax}}

\DeclareMathOperator*{\argmin}{argmin}

\newcommand{\BR}{\mathbbm{1}}
\newcommand{\E}{\mathbb E}
\newcommand{\PP}{\mathbb P}

\newcommand{\Hy}{\mathcal{H}}
\newcommand{\A}{\textsf{A}}
\newcommand{\FNR}{\textsf{FNR}}
\newcommand{\FPR}{\textsf{FPR}}

\newcommand{\f}{f^*}
\newcommand{\rf}{f}

\newcommand{\D}{\mathcal D}

\newcommand{\squishlist}{
   \begin{list}{$\bullet$}
    { \setlength{\itemsep}{0pt}      \setlength{\parsep}{3pt}
      \setlength{\topsep}{3pt}       \setlength{\partopsep}{0pt}
      \setlength{\leftmargin}{1.5em} \setlength{\labelwidth}{1em}
      \setlength{\labelsep}{0.5em} } }
\newcommand{\squishend}{  \end{list}  }

\ifodd 1

\newcommand{\yl}[1]{\textbf{\color{red}(Yang: #1)}}

\newcommand{\clar}[1]{\textbf{\color{green}(NEED CLARIFICATION: #1)}}
\newcommand{\response}[1]{\textbf{\color{magenta}(RESPONSE: #1)}}
\else

\newcommand{\yl}[1]{}
\newcommand{\clar}[1]{}
\newcommand{\response}[1]{}
\fi

\title{Incentives for Federated Learning: a Hypothesis Elicitation Approach}

\author{
 Yang Liu\\
  UC Santa Cruz\\
  \texttt{yangliu@ucsc.edu} \\
   \And
 Jiaheng Wei \\
  UC Santa Cruz\\
  \texttt{jiahengwei@ucsc.edu} \\
}

\begin{document}
\maketitle

\begin{abstract}
Federated learning provides a promising paradigm for collecting machine learning models from distributed data sources without compromising users' data privacy. The success of a credible federated learning system builds on the assumption that the decentralized and self-interested users will be willing to participate to contribute their local models in a trustworthy way. However, without proper incentives, users might simply opt out the contribution cycle, or will be mis-incentivized to contribute spam/false information.  This paper introduces solutions to incentivize truthful reporting of a local, user-side machine learning model for federated learning. Our results build on the literature of information elicitation, but focus on the questions of \emph{eliciting hypothesis} (rather than eliciting human predictions). We provide a scoring rule based framework that incentivizes truthful reporting of local hypotheses at a Bayesian Nash Equilibrium. We study the market implementation, accuracy as well as robustness properties of our proposed solution too. We verify the effectiveness of our methods using MNIST and CIFAR-10 datasets. Particularly we show that by reporting low-quality hypotheses, users will receive decreasing scores (rewards, or payments).
\end{abstract}

% keywords can be removed

\section{Introduction}

When a company relies on distributed users' data to train a machine learning model, federated learning \cite{mcmahan2016communication,yang2019federated,kairouz2019advances} promotes the idea that users/customers' data should be kept local, and only the locally held/learned hypothesis will be shared/contributed from each user. While federated learning has observed success in keyboard recognition \cite{hard2018federated} and in language modeling \cite{chen2019federated}, existing works have made an implicit assumption that participating users will be willing to contribute their local hypotheses to help the central entity to refine the model. Nonetheless, without proper incentives, agents can choose to opt out of the participation, to contribute either uninformative or outdated information, or to even contribute malicious model information. Though being an important question for federated learning \cite{yang2019federated,liu2020fedcoin,hansus,han20}, this capability of providing adequate incentives for user participation has largely been overlooked. In this paper we ask the questions that: \emph{Can a machine learning hypothesis be incentivized/elicited by a certain form of scoring rules from self-interested agents?} The availability of a scoring rule will help us design a payment for the elicited hypothesis properly to motivate the reporting of high-quality ones. The corresponding solutions complement the literature of federated learning by offering a generic template for incentivizing users' participation. 

We address the challenge via providing a scoring framework to elicit hypotheses truthfully from the self-interested agents/users\footnote{Throughout this paper, we will interchange the use of agents and users.}. More concretely, suppose an agent $i$ has a locally observed hypothesis $\f_i$. For instance, the hypothesis can come from solving a local problem:
%Each agent believes that his own classifier is the Bayes optimal classifier that minimizes the 0-1 risk:
   $
    \f_i = \argmin_{f_i \sim \mathcal{H}_i} \mathbb{E}_{(X,Y) \sim \mathcal D} [\ell_i \left(f_i(X),
    Y\right) ]
   $
according to a certain hypothesis class $\Hy_i$, a distribution $\mathcal D$, a loss function $\ell_i$. The goal is to design a scoring function $S(\cdot)$ that takes a reported hypothesis $\rf_i$, and possibly a second input argument (to be defined in the context) such that
$
\E\left[S(\f_i, \cdot)\right] \geq \E\left[S(\rf_i, \cdot)\right],\forall \rf_i 
$, where the expectation is w.r.t. agent $i$'s local belief, which is specified in context. If the above can be achieved, $S(\cdot)$ can serve as the basis of a payment system in federated learning such that agents paid by $S(\cdot)$ will be incentivized to contribute their local models truthfully. In this work, we primarily consider two settings, with arguably increasing difficulties in designing our mechanisms:

% \vspace{-0.2in}
\paragraph{With ground truth verification $(X,Y)$} We will start with a relatively easier setting where we as the designer has access to a labeled dataset $\{(x_n,y_n)\}_{n=1}^N$. We will demonstrate how this question is similar to the classical information elicitation problem with strictly proper scoring rule \cite{gneiting2007strictly}, calibrated loss functions \cite{bartlett2006convexity} and peer prediction (information elicitation without verification) \cite{miller2005eliciting}.

%Each agent believes that his own classifier is the Bayes optimal classifier that minimizes the 0-1 risk:
%    \[
%    f_i = \argmin_{f_i \sim \mathcal{H}} \mathbb{E}_{(X,Y)} \bigg[\mathbbm{1} \Big(f_i(X) \neq
%    Y\Big) \bigg]
%    \]
%    Then probably from here, any classification calibrated loss functions can serve as a scoring function.
   
%   We also consider the casebut each agent not necessarily believe their classifier is maximizing/minmizing the Bayes risk. In this setting, let's denote the ground truth $Y$ also as a hypothesis $f^*(X)$. Then we can apply a peer prediction scoring function that takes $f_i(X)$ and $f^*(X)$ as inputs. For instance, if we follow the correlated agreement mechanism and all we need to know is the marginal agreement matrix $\Delta(f_i,f^*)$ defined between $f_i$ and $f^*$. 
    %\item 
    % \vspace{-0.2in}

\paragraph{With only access to features $X$}    The second setting is when we only have $X$ but not the ground truth $Y$. This case is arguably more popular in practice, since collecting label annotation requires a substantial amount of efforts. For instance, a company is interested in eliciting/training a classifier for an image classification problem. While it has access to images, it might not have spent efforts in collecting labels for the images. We will again present a peer prediciton-ish solution for this setting. %show that how this setting can be easily translated to the standard peer prediction setting. %: $S(f_i(X), f_j(X))$. 

Besides establishing the desired incentive properties of the scoring rules, we will look into questions such as when the scoring mechanism is rewarding accurate classifiers, how to build a prediction market-ish solution to elicit improving classifiers, as well as our mechanism's robustness against possible collusion. 
Our work can be viewed both as a contribution to federated learning via providing incentives for selfish agents to share their hypotheses, as well as a contribution to the literature of information elicitation via studying the problem of hypothesis elicitation. We validate our claims via experiments using the MNIST and CIFAR-10 datasets.
    %\item
    
%\paragraph{Incentivizing differentially private hypothesis}    The last setting that is of interest is incentivizing differentially private hypothesis. In response to privacy attacks using shared hypothesis \cite{}, recent works proposed differentially private algorithms for sharing machine learning models. We extend our results to eliciting differentially private hypothesis. 

All omitted proofs and experiment details can be found in the supplementary materials.

\subsection{Related works}
Due to space limit, we only briefly survey the related two lines of works:

% \vspace{-0.2in}
\paragraph{Information elicitation} Our solution concept relates most closely to the literature of information elicitation \cite{Brier:50,Win:69,Savage:71,Matheson:76,Jose:06,Gneiting:07}. Information elicitation primarily focuses on the questions of developing scoring rule to incentivize or to elicite self-interested agents' private probalistic beliefs about a private event (e.g., how likely  will COVID-19 death toll reach 100K by May 1?). Relevant to us, \cite{abernethy2011collaborative} provides a market treatment to elicit more accurate classifiers but the solution requires the designer to have the ground truth labels and agents to agree on the losses. We provide a more generic solution without above limiations.

A more challenging setting features an elicitation question while there sans ground truth verification. Peer prediction \cite{Prelec:2004,MRZ:2005,witkowski2012robust,radanovic2013,Witkowski_hcomp13,dasgupta2013crowdsourced,shnayder2016informed,radanovic2016incentives,LC17,kong2019information,liu2020} is among the most popular solution concept. The core idea of peer prediction is to score each agent based on another reference report elicited from the rest of agents, and to leverage on the stochastic correlation between different agents' information. Most relevant to us is the Correlated Agreement mechanism \cite{dasgupta2013crowdsourced,shnayder2016informed,kong2019information}. We provide a separate discussion of it in Section \ref{sec:pp}.% This line of research includes~\cite{Prelec:2004,MRZ:2005,witkowski2012robust,radanovic2013,Witkowski_hcomp13,dasgupta2013crowdsourced,shnayder2016informed,radanovic2016incentives,LC17,kong2019information,liu2020}. 

% \vspace{-0.2in}
\paragraph{Federated learning} Federated learning \cite{mcmahan2016communication,hard2018federated,yang2019federated} arose recently as a promising architecture for learning from massive amounts of users' local information without polling their private data. The existing literature has devoted extensive efforts to make the model sharing process more secure \cite{secure_1, secure_2, secure_3, secure_4, secure_5, bonawitz2016practical}, more efficient \cite{efficient_1,efficient_2,efficient_3,efficient_4,fl:communication, efficient_6, efficient_7}, more robust \cite{robust_1,robust_2,robust_3,pillutla2019robust} to heterogeneity in the distributed data source, among many other works. For more detailed survey please refer to several thorough ones \cite{yang2019federated,kairouz2019advances}.

The incentive issue has been listed as an outstanding problem in federated learning \cite{yang2019federated}. 
There have been several very recent works touching on the challenge of incentive design in federated learning. \cite{liu2020fedcoin} proposed a currency system for federated learning based on blockchain techniques. \cite{hansus} describes a payoff sharing algorithm that maximizes system designer's utility, but the solution does not consider the agents' strategic behaviors induced by insufficient incentives. \cite{han20} further added fairness guarantees to an above reward system. We are not aware of a systematic study of the truthfulness in incentiving hypotheses in federated learning, and our work complements above results by providing an incentive-compatible scoring system for building a payment system for federated learning.

\section{Formulation}
Consider the setting with a set $\mathcal{K} = \{ 1, 2, ..., K \}$ of agents, each with a
hypothesis $\f_i \in \mathcal{H}_i$ which maps feature space $X$ to label
space $Y \in \{1, 2,...,L\}:=[L]$. The hypothesis space $\mathcal{H}_i$ is the space of hypotheses accessible or yet considered by agent $i$, perhaps as a function of the subsets of $X$ or $Y$ which have been encountered by $i$ or the agent's available computational power. 
$\f_i$ is often obtained following a local optimization process. For example, % given some \textit{true} function $f^*: X \to Y$, 
 $\f_i$ can be defined as the
function which minimizes a loss function over an agent's hypothesis space.
\[
\f_i = \argmin_{f_i \sim \mathcal{H}_i} \mathbb{E}_{\mathcal D_i} \left[\mathbbm{1} \Big(f_i(X) \neq
  Y\Big)~\right]
  \]
where in above $\D_i$ is the local distribution that agent $i$ has access to train and evaluate $\f_i$.
In the federated learning setting, note that $\f_i$ can also represent the optimal output from a private training algorithm and $\mathcal{H}_i$ would denote a training hypothesis space that encodes a certainly level of privacy  guarantees. In this paper, we do not discuss the specific ways to make a local hypothesis private \footnote{There exists a variety of definitions of privacy and their corresponding solutions for achieving so. Notable solutions include \emph{output perturbation} \cite{chaudhuri2011differentially} or \emph{output sampling} \cite{bassily2014private} to preserve privacy when differential privacy \cite{dwork2006differential} is adopted to quantify the preserved privacy level.}, but rather we focus on developing scoring functions to incentivize/elicit this ``private" and ready-to-be shared hypothesis.

 %Note that $\f_i$ can also be the optimal output from a private training algorithm. For example, in implementing a federated learning system, the principle often collects a set of hypotheses with privacy guarantees. There exists a variety of definitions of privacy and their corresponding solutions for achieving so. 

Suppose the mechanism designer has access to a dataset $D$: $D$ can be a standard training set with pairs of features and labels $D:=\{(x_n,y_n)\}_{n=1}^N$, or we are in a unsupervised setting where we don't have labels associated with each sample $x_i$: $D:=\{x_n\}_{n=1}^N$.

The \emph{goal} of the mechanism designer is to collect $\f_i$ truthfully from agent $i$. Denote the reported/contributed hypothesis from agent $i$ as $f_i$\footnote{$f_i$ can be \texttt{none} if users chose to not contribute.}. Each agent will be scored using a function $S$ that takes all reported hypotheses $f_j, \forall j$ and $D$ as inputs:
$
S\Big(f_i, \{f_{j \neq i} \}, D\Big)
$
such that it is ``proper'' at a Bayesian Nash Equilibrium:
\begin{definition}
$S(\cdot)$ is called inducing truthful reporting at a Bayesian Nash Equilibrium if for every agent $i$, assuming for all $j \neq i$, $f_j = \f_j$ (i.e., every other agent is willing to report their hypotheses
truthfully),
\[
\mathbb{E}\left[S\Big(\f_i, \{\f_{j \neq i} \}, D\Big) \right] \geq
\mathbb{E}\left[S\Big(f_i,\{\f_{j \neq i} \}, D\Big) \right],~~~~\forall f_i, %\neq \f_i
\]
where the expectation encodes agent $i$'s belief about $\{\f_{j \neq i} \}~\text{and}~ D$. % and in above context 
%$
% f_i \neq \f_i \Leftrightarrow \PP( f_i(X) \neq \f_i(X)) > 0.
%$
\end{definition}

\subsection{Peer prediction}
\label{sec:pp}
%Very relevant to us is strictly proper scoring rules
%$
%S(p,o): [0,1]\times Y \rightarrow \mathbb R,
%$
%which is defined to score a probabilistic prediction $p$ about an event' outcome $o$. $S(\cdot)$ is called strictly proper if
%$
%\E_{o \sim p} [S(p,o)] > \E_{o \sim p} [S(q,o)], ~\forall q \neq p.
%$ Proper scoring rules have primarily focused on scoring and eliciting human probabilistic forecasts, but not a machine learning predictor. 

%\subsection{Peer prediction}

\emph{Peer prediction} is a technique developed to truthfully elicit information when there is no ground truth verification.  
Suppose we are interested in eliciting private observations about a categorical event $y \in [L]$ generated according to a random variable $Y$ (in the context of a machine learning task, $Y$ can be thought of as labels). Each of the $K \geq 2$ agents holds a noisy observation of $y$, denoted as $y_i \in [L],\, i \in [K]$. Again the goal of the mechanism designer is to elicit the $y_i$s, but they are private and we do not have access to the ground truth $Y$ to perform an evaluation. %Denote by $r_i$ the reported data from each agent $i$. Results in \emph{peer prediction} have proposed scoring or reward functions that evaluate an agent's report using the reports of other peer agents. For example, a peer prediction mechanism may reward agent $i$ for her report $r_i$ using $S(r_i, r_{j \neq i})$ where $r_{j \neq i}$ is the report of a randomly selected reference agent $j \in [K]\backslash \{i\}$.
The scoring function $S$ is designed so that truth-telling is a strict Bayesian Nash Equilibrium (implying other agents truthfully report their $y_j$), that is,
 $\forall i$, $
\mathbb E_{y_j}\left[S\left( y_i, y_j \right)|y_i\right] > 
\mathbb E_{y_j}\left[S\left(r_i, y_j\right)| y_i\right],~\forall r_i \neq y_i.
$
%There has been a line of developments for peer prediction mechanisms. %Below we will introduce three recent ones which we will also highlight 

\paragraph{Correlated Agreement} 
Correlated Agreement (CA) \cite{dasgupta2013crowdsourced,2016arXiv160303151S}  is a recently established peer prediction mechanism for a multi-task setting. CA is also the core and the focus of our subsequent sections. This mechanism builds on a $\Delta$ matrix that captures the stochastic correlation between the two sources of predictions $y_i$ and $y_j$.
%Denote the following mapping function: $g(1) = -1, g(2) = +1$, 
$\Delta \in \mathbb R^{L \times L}$ is then defined as a squared matrix with its entries defined as follows:
\[
    \Delta(k,l)  = \PP\bigl(y_i=k,y_j=l\bigr)- \PP\bigl(y_i = k\bigr) \PP\bigl(y_j= l\bigr), ~k,l \in [L].
\]
The intuition of above $\Delta$ matrix is that each $(i,j)$ entry of $\Delta$ captures the marginal correlation between the two predictions.  %$M\in \mathbb R^{2 \times 2}$ is then defined  
$Sgn(\Delta)$ denotes the sign matrix of $\Delta$:$~\text{where}~Sgn(x)=1, x > 0; ~Sgn(x)=0, \text{o.w.}$ %Define the following score matrix $M^S: \{-1,+1\} \times \{-1,+1\} \rightarrow \{0,1\}:$
%\begin{align}
%M^S(y, y') =:  M\left(y,y'\right), \label{eqn:ms}
%\end{align}
 %where $g^{-1}$ is the inverse function of $g$. 

 CA requires each agent $i$ to perform multiple tasks: denote agent $i$'s observations for the $N$ tasks as $y_{i,1},...,y_{i,N}$. 
Ultimately the scoring function $S(\cdot)$ for each task $k$ that is shared between $i,j$ is defined as follows: randomly draw two other tasks $k_1,k_2~, k_1 \neq k_2 \neq k$,
\begin{align*}
S\bigl(y_{i,k},y_{j,k}\bigr) :=& Sgn\bigl(\Delta(y_{i,k}, y_{j,k})\bigr) - Sgn\bigl(\Delta(y_{i,k_1},y_{j,k_2})\bigr), 
\end{align*}
%Note a key difference between the first and second $M$ terms is that the second term is defined for two independent peer tasks $k_1,k_2$ (as the reference answers). 
It was established in \cite{2016arXiv160303151S} that CA is truthful and proper (Theorem 5.2,  \cite{2016arXiv160303151S}) \footnote{To be precise, it is an informed truthfulness. We refer interested readers to \cite{2016arXiv160303151S} for the detailed differences.}. % in particular, if $y_j$ is \emph{categorical} w.r.t. $y_i$:
$
\PP(y_j=y'|y_i = y) < \PP(y_j=y'), \forall i,j \in [K],~y' \neq y
$
then $S(\cdot) $ is strictly truthful (Theorem 4.4,  \cite{2016arXiv160303151S}). 

%\subsubsection{Surrogate scoring rules}
%Surrogate scoring rules are a recent set of mechanisms proposed to remove biases from the reference answer in scoring when the ground truth verification is missing. 
%Consider a binary signal case when $L=2$. Denote by the following error rate model
%\[
%\alpha:=\PP(y=2|y_j = 1),~~\beta:=\PP(y=1|y_j=2)
%\]
%Then suppose $S'$ is a score function the elicits $y_i$ truthfully using ground truth $y$. Then 
%\begin{align*}
%S(y_i,y_j=1) = \frac{(1-\beta)\cdot S'(y_i,1)-\alpha\cdot S'(y_i,2)}{1-\alpha-\beta}\\
%S(y_i,y_j=2) = \frac{(1-\alpha) \cdot S'(y_i,2)-\beta \cdot S'(y_i,1)}{1-\alpha-\beta}
%\end{align*}
%$S'$ can take as 1/Prior mechanism or any peer prediction mechanism defined to elicit $y_i$ truthfully with ground truth $y$.

%Surrogate scoring rules prove to be robust to agents' manipulations but require an estimation procedure to know $\alpha$ and $\beta$.

\section{Elicitation with verification}

We start by considering the setting where the mechanism designer has access to ground truth labels, i.e., $D=\{(x_n,y_n)\}_{n=1}^N$. %

%following setting such that i) each agent believes their local hypothesis is optimal within a certain hypothesis class (could be any hypothesis space); ii) the mechanism designer has access to ground truth samples $(X,Y)$.

\subsection{A warm-up case: eliciting Bayes optimal classifier}
As a warm-up, we start with the question of eliciting the Bayes optimal classifier: 
\[
\f_i = \argmin_{f_i} \mathbb{E}_{(X,Y)} \left[\mathbbm{1} \Big(f_i(X) \neq
    Y\Big) \right].
\]
It is straightforward to observe that, by definition using $-\BR(\cdot)$ (negative sign changes a loss to a reward (score)) and any affine transformation of it $a \BR(\cdot) + b, a<0$ will be sufficient to incentivize truthful reporting of hypothesis. Next we are going to show that any classification-calibrated loss function \cite{bartlett2006convexity} can serve as a proper scoring function for eliciting hypothesis.\footnote{We provide details of the calibration in the proof. Classical examples include cross-entropy loss, squared loss, etc.}

%We prove the following 
\begin{theorem}\label{thm:calibrate1}
Any classification calibrated loss function $\ell(\cdot)$ (paying agents $-\ell(f_i(X), Y)$) induces truthful reporting of the Bayes optimal classifier.
\end{theorem}

\subsection{Eliciting ``any-optimal" classifier: a peer prediction approach}

Now consider the case that an agent does not hold an absolute Bayes optimal classifier. Instead, in practice, agent's local hypothesis will depend on the local observations they have, the privacy level he desired, the hypothesis space and training method he is using. Consider agent $i$ holds the following hypothesis $\f_i$, according to a loss function $\ell_i$, and a hypothesis space $\Hy_i$:
$
\f_i = \argmin_{f_i \sim \mathcal{H}_i} \mathbb{E}\left[\ell_i \Big(f_i(X),Y\Big) \right].
$

By definition, each specific $\ell_i$ will be sufficient to incentivize a hypothesis. However, it is unclear how $\f_i$ trained using $\ell_i$ would necessarily be optimal according to a universal metric/score. We aim for a more generic approach to elicit different $\f_i$s that are returned from different training procedure and hypothesis classes. In the following sections, we provide a peer prediction approach to do so.

%\subsubsection{Peer prediction for hypothesis elicitation}

We first state the hypothesis elicitation problem as a standard peer prediction problem. The connection is made by firstly rephrasing the two data sources, the classifiers and the labels, from agents' perspective.  Let's re-interpret the ground truth label $Y$ as an ``optimal" agent who holds a hypothesis $f^*(X) = Y$. %For a task $y \in [L]$, say $l$ for example, denote the labels $Y$ as $f^*(X), X \sim \PP_{X|Y=y}$. W
We denote this agent as $\A^*$. Each local hypothesis $\f_i$ agent $i$ holds can be interpreted as the agent that observes $\f_i(x_1),...,\f_i(x_N)$ for a set of randomly drawn feature vectors $x_1,...,x_N$:
$
\f_i(x_n) \sim \A_i(X).
$ Then a peer prediction mechanism induces truthful reporting if:
$
\E\left[S(\f_i(X), f^*(X))\right] \geq \E\left[S(f(X), f^*(X))\right],~\forall f. % \neq \f_i(X)
$
%Formally we have the following theorem:
%\begin{theorem}
%A peer prediction scoring function $S$ that can elicit truthful reports from $\A_i$ using $\A^*$ will elicit $\f_i$ truthfully.
%\end{theorem}
%The above theorem is s
\paragraph{Correlated Agreement for hypothesis elicitation} To be more concrete, consider a specific implementation of peer prediction mechanism, the Correlated Agreement (CA). Recall that the mechanism builds on a correlation matrix $\Delta(\f_i(X),f^*(X))$ defined as follows:
\begin{align*}
    &\Delta^*(k,l)  = \PP\bigl(\f_i(X)=k,f^*(X)=l\bigr)- \PP\bigl(\f_i(X) = k\bigr) \PP\bigl(f^*(X)= l\bigr), ~k,l \in [L].
\end{align*}
%Then define $M$ as the sgn matrix of $\Delta^*$:
%\[
%M(i,j) = Sgn\left(\Delta^*(i,j)\right)
%\]
%where $sgn(x) = 1$ if $x \geq 0$, and $0$ otherwise. The mechanism operates as follows:
Then the CA for hypothesis elicitation is summarized in Algorithm \ref{m:main1}.
\begin{algorithm}
\caption{CA for Hypothesis Elicitation}\label{m:main1}
\begin{algorithmic}[1]
\STATE For each sample $x_n$, randomly sample two other tasks $x_{p_1} \neq x_{p_2} \neq x_n$ to pair with.
\STATE Pay a reported hypothesis $f(\cdot)$ for $x_n$ according to 
    \begin{align}
    S(f(x_n),f^*(x_n)):=Sgn\left(\Delta^*(f(x_n),f^*(x_n))\right)-Sgn\left(\Delta^*(f(x_{p_1}),f^*(x_{p_2}))\right)%\right)
    \end{align}
\STATE Total payment to a reported hypothesis $f$:
    $
    S(f,\f) := \sum_{n=1}^N S(f(x_n),f^*(x_n)).$
\end{algorithmic}
\end{algorithm}

We reproduce the incentive guarantees and required conditions:% (we drop the subscript $i$ from $f_i$ since it holds for any hypothesis $f$)
\begin{theorem}\label{thm:CA:truthful}
CA mechanism induces truthful reporting of a hypothesis at a Bayesian Nash Equilibrium.
\end{theorem}

\paragraph{Knowledge requirement of $\Delta^*$} We'd like to note that knowing the sign of $\Delta^*$ matrix between $\f_i$ and $f^*$ is a relatively weak assumption to have to run the mechanism. For example, for a binary classification task $L=2$, define the following accuracy measure, 
\[
\FNR(f) := \PP(f(X)=2|Y=1),~\FPR(f) := \PP(f(X)=1|Y=2).
\]
We offer the following:
\begin{lemma}\label{lemma:delta}
For binary classification ($L=2$), if $\FNR(\f_i) + \FPR(\f_i) < 1$, $Sgn(\Delta^*)$ is an identify matrix. 
\end{lemma}
$\FNR(\f_i) + \FPR(\f_i) < 1$ is stating that $\f_i$ is informative about the ground truth label $Y$ \cite{LC17}. Similar conditions can be derived for $L>2$ to guarantee an identify $Sgn(\Delta^*)$.
With identifying a simple structure of $Sgn(\Delta^*)$, the CA mechanism for hypothesis elicitation runs in a rather simple manner. 

\paragraph{When do we reward accuracy} 
The elegance of the above CA mechanism leverages the correlation between a classifier and the ground truth label. Ideally we'd like a mechanism that rewards the accuracy of the contributed classifier. Consider the binary label case:
\begin{theorem}\label{thm:accuracy}
When $\PP(Y=1) = 0.5$ (uniform prior), and let $Sgn(\Delta^* = I_{2\times 2})$ be the identity matrix, the more accurate classifier within each $\mathcal H_i$ receives a higher score. 
\end{theorem}
Note that the above result does not conflict with our incentive claims. In an equal prior case, misreporting can only reduce a believed optimal classifier's accuracy instead of the other way. It remains an interesting question to understand a more generic set of conditions under which CA will be able to incentivize contributions of more accurate classifiers. 

\vspace{-0.05in}
\paragraph{A market implementation} The above scoring mechanism leads to a market implementation \cite{hanson2007logarithmic} that incentivizes improving classifiers. In particular, suppose agents come and participate at discrete time step $t$. Denote the hypothesis agent contributed at time step $t$ as $\f_t$ (and his report $\rf_t$). Agent at time $t$ will be paid according to
$S(\rf_{t}(X),Y) - S(\rf_{t-1}(X),Y),~ %\label{eqn:market}
$
where $S(\cdot)$ is an incentive-compatible scoring function that elicits $f_t$ truthfully using $Y$. The incentive-compatibility of the market payment is immediate due to $S(\cdot)$. The above market implementation incentivizes improving classifiers with bounded budget \footnote{Telescoping returns: $
\sum_{t=1}^T \left(S(f_{t}(X),Y) - S(f_{t-1}(X),Y) \right) = S(f_T(X),Y)-S(f_0(X),Y)$.}. 
%This result echoes previous design as presented in \cite{abernethy2011collaborative}.

%features two favorable properties:
%\squishlist
%    \item \emph{Incentivizing improving classifiers}: due to the individual rationality, agent $t$ will only chose to contribute when he believes his classifier $\f_t$ is better than the current ``market belief" $\rf_{t-1}$.
%    \item \emph{Bounded budget:} A second favorable property of a market implementation is to provide a budget-balanced mechanism:
%\[
%\sum_{t=1}^T \left(S(f_{t}(X),Y) - S(f_{t-1}(X),Y) \right) = S(f_T(X),Y)-S(f_0(X),Y) 
%\]
%\squishend

\paragraph{Calibrated CA scores\label{sec:reward sturcture}}

%A recent pre-print studied the calibration property via replacing the 0-1 loss with an arbitrary one \citep{}. 
%Our previous section provides results for calibrated loss functions. 
When $\Delta^*$ is the identity matrix, the CA mechanism reduces to:
$$
S(f(x_n),f^*(x_n)):=\BR\left(f(x_n) = f^*(x_n)\right)-\BR\left(f(x_{p_1})=f^*(x_{p_2})\right)
$$
That is the reward structure of CA builds on 0-1 loss function. We ask the question of can we extend the CA to a calibrated one? We define the following loss-calibrated scoring function for CA:
%\begin{align}
\[\textsf{Calibrated CA:~~~}
S_{\ell}(f(x_n),f^*(x_n)) = -\ell(f(x_n),f^*(x_n)))-(-\ell(f(x_{p_1}),f^*(x_{p_2}))).~\label{ca:calibrated}
\]
%\end{align}
%We would like to understand the incentive property provided by above. 
Here again we negate the loss $\ell$ to make it a reward (agent will seek to maximize it instead of minimizing it). If this extension is possible, not only  we will be able to include more scoring functions, but also we are allowed to score/verify non-binary classifiers directly. Due to space limit, we provide positive answers and detailed results in Appendix, while we will present empirical results on the calibrated scores of CA in Section \ref{sec:exp}.

\section{Elicitation without verification}

Now we move on to a more challenging setting where we do not have ground truth label $Y$ to verify the accuracy, or the informativeness of $f(X)$, i.e., the mechanism designer only has access to a $D=\{x_n\}_{n=1}^N$. % %This makes our problem fall into the standard peer prediction setting, where the principle aims to elicit information but without verification.
The main idea of our solution from this section follows straight-forwardly from the previous section, but instead of having a ground truth agent $f^*$, for each classifier $\f_i$ we only have a reference agent $\f_j$ drawn from the rest agents $j \neq i$ to score with. The corresponding scoring rule takes the form of $S(\rf_i(X),\rf_j(X))$, and similarly the goal is to achieve the following:
$
\E\left[S(\f_i(X), \f_j(X))\right] \geq \E\left[S(f(X), \f_j(X))\right],~\forall f. %: \PP(x \in \{x: f(x) \neq \f_i(x)\}) > 0
$

As argued before, if we treat $f_i$ and $f_j$ as two agents $\A_i$ and $\A_j$ holding private information, a properly defined peer prediction scoring function that elicits $\A_i$ using $\A_j$ will suffice to elicit $f_i$ using $f_j$. Again we will focus on using Correlated Agreement as a running example. Recall that the mechanism builds on a correlation matrix $\Delta^*(\f_i(X),\f_j(X))$.
\begin{align*}
    &\Delta^*(k,l)  = \PP\bigl(\f_i(X)=k,\f_j(X)=l\bigr)- \PP\bigl(\f_i(X) = k\bigr) \PP\bigl(\f_j(X)= l\bigr), ~k,l \in [L]
\end{align*}
%and its sign matrix $M$. 
The mechanism then operates as follows: For each task $x_n$, randomly sample two other tasks $x_{p_1},x_{p_2}$. Then pay a reported hypothesis according to 
\vspace{-0.05in}
    \begin{align}
        S(\rf_i,\rf_j):= \sum_{n=1}^N Sgn\left(\Delta^*(\rf_i(x_n),\rf_j(x_n))\right)-Sgn\left(\Delta^*(\rf_i(x_{p_1}),\rf_j(x_{p_2}))\right)
    \end{align}
%\squishend
We reproduce the incentive guarantees:
\begin{theorem}
CA mechanism induces truthful reporting of a hypothesis at a Bayesian Nash Equilibrium.
\end{theorem}
The proof is similar to the proof of Theorem \ref{thm:CA:truthful} so we will not repeat the details in the Appendix.

To enable a clean presentation of analysis, the rest of this section will focus on using/applying CA for the binary case $L=2$. First, as an extension to Lemma \ref{lemma:delta}, we have:
\begin{lemma}
\label{col:delta}
If $\f_i$ and $\f_j$ are conditionally independent given $Y$, $\FNR(\f_i) + \FPR(\f_i) < 1$ and $\FNR(\f_j) + \FPR(\f_j) < 1$, then $Sgn(\Delta^*)$ is an identify matrix. 
\end{lemma}

\paragraph{When do we reward accuracy}

As mentioned earlier that in general peer prediction mechanisms do not incentivize accuracy. Nonetheless we provide conditions under which they do. The result below holds for binary classifications.

\begin{theorem}\label{thm:pp:accuracy}
When (i) $\PP(Y=1) = 0.5$, (ii) $Sgn(\Delta^*) = I_{2 \times 2}$, and (iii) $\f_i(X)$ and  $\f_j(X)$ are conditional independent of $Y$, the more accurate classifier within each $\mathcal H_i$ receives a higher score in expectation. 
\end{theorem}

\subsection{Peer Prediction market}

Implementing the above peer prediction setting in a market setting is hard, due to again the challenge of no ground truth verification. The use of reference answers collected from other peers to similarly close a market will create incentives for further manipulations. 

Our first attempt is to crowdsource to obtain an independent survey answer and use the survey answer to close the market. Denote the survey hypothesis as $ f'$ and use $f'$ to close the market:
\begin{align}\label{eqn:crowd:market}
    S(\rf_{t}(x),f'(x)) - S(\rf_{t-1}(X),f'(x))
\end{align}
\begin{theorem}\label{thm:market:pp1}
When the survey hypothesis $f'(x)$ is (i) conditionally independent from the market contributions, and (ii) Bayesian informative, %[Definition \ref{def:BI}, replacing $f$ with $f'$], 
then closing the market using the crowdsourcing survey hypothesis is incentive compatible. 
\end{theorem}

The above mechanism is manipulable in several aspects. Particularly, the crowdsourcing process needs to be independent from the market, which implies that the survey participant will need to stay away from participating in the market - but it is unclear whether this will be the case. In the Appendix we show that by maintaining a survey process that elicits $C>1$ hypotheses, we can further improve the robustness of our mechanisms against agents performing a joint manipulation on both surveys and markets.

\paragraph{Remark} Before we conclude this section, we remark that the above solution for the \emph{without verification} setting also points to an hybrid solution when the designer has access to both sample points with and without ground truth labels. The introduction of the pure peer assessment solution helps reduce the variance of payment.

\subsection{Robust elicitation}
 
Running a peer prediction mechanism with verifications coming only from peer agents is vulnerable when facing collusion. In this section we answer the question of how robust our mechanisms are when facing a $\gamma$-fraction of adversary in the participating population. To instantiate our discussion, consider the following setting
\squishlist
    \item There are $1-\gamma$ fraction of agents who will act truthfully if incentivized properly. Denote the randomly drawn classifier from this $1-\gamma$ population as $\f_{1-\gamma}$.
    \item There are $\gamma$ fraction of agents are adversary, whose reported hypotheses can be arbitrary and are purely adversarial. 
\squishend
Denote the following quantifies
$
    \alpha := \PP(\f_{1-\gamma}(X)=2|Y=1),~\beta := \PP(\f_{1-\gamma}(X)=1|Y=2)~\alpha^* := \PP(\f(X)=2|Y=1),~\beta^* := \PP(\f(X)=1|Y=2)
$, that is $\alpha,\beta$ are the error rates for the eliciting classifier $\f_{1-\gamma}$ while $\alpha^*,\beta^*$ are the error rates for the Bayes optimal classifier. We prove the following 
\begin{theorem}\label{thm:robust}
CA is truthful in eliciting hypothesis when facing $\gamma$-fraction of adversary when $\gamma$ satisfies: %the following:
$
\frac{1-\gamma}{\gamma} > \frac{1-\alpha^*-\beta^*}{1-\alpha-\beta}. 
$
\end{theorem}
When the agent believes that the classifier the $1-\gamma$ crowd holds is as accurate as the Bayes optimal classifier we have
$
\frac{1-\alpha^*-\beta^*}{1-\alpha-\beta} = \frac{1-\alpha^*-\beta^*}{1-\alpha^*-\beta^*} = 1
$, then a sufficient condition for eliciting truthful reporting is $\gamma < 50\%$, that is our mechanism is robust up to half of the population manipulating. Clearly the more accurate the reference classifier is, the more robust our mechanism is.

%\newpage

%, per the following theorem
%\begin{theorem}
%By assigning each agent $N$ non-ground truth ones and scoring them using our peer prediction approach, 
%\end{theorem}

\section{Experiments}\label{sec:exp}

In this section, we implement two reward structures of CA: 0-1 score and Cross-Entropy (CE) score as mentioned at the end of Section~\ref{sec:reward sturcture}. We experiment on two image classification tasks: MNIST~\cite{mnist} and CIFAR-10~\cite{cifar} in our experiments. For agent $\A_{W}$ (weak agent), we choose LeNet~\cite{mnist} and ResNet34~\cite{resnet} for MNIST and CIFAR-10 respectively. For $\A_{S}$ (strong agent), we use a 13-layer CNN architecture for both datasets.

% We consider two user-side machine learning models: LeNet~\cite{mnist} and ResNet34~\cite{resnet} as agent $\A_{W}$ (weak agent) and $\A_{S}$ (strong agent).
Either of them is trained on random sampled 25000 images from each image classification training task. After the training process, agent $\A_{W}$ reaches 99.37\% and 62.46\% test accuracy if he truthfully reports the prediction on MNIST and CIFAR-10 test data. Agent $\A_{S}$ is able to reach 99.74\% and 76.89\% test accuracy if the prediction on MNIST and CIFAR-10 test data is truthfully reported. 

$\A_{W}$ and $\A_{S}$ receive hypothesis scores based on the test data $X_{\text{test}}$ (10000 test images) of MNIST or CIFAR-10. For elicitation with verification, we use ground truth labels to calculate the hypothesis score. For elicitation without verification, we replace the ground truth labels with the other agent's prediction - $\A_{W}$ will serve as $\A_{S}$'s peer reference hypothesis and vice versa. 

% We focus on three kinds of misreport models: uniform off-diagonal models, sparse models and adversarial models.
% \yl{the misreporting should be defined for $\pp(\tilde{f}_i=j|\f_i = i)$, i.e., the probability of flipping the truthful classifier}

\subsection{Results} \label{sec:exp_1}
Statistically, an agent $i$'s mis-reported hypothesis can be expressed by a misreport transition matrix $T$. Each element $T_{j,k}$ represents the probability of flipping the truthfully reported label $\f_i(x) = j$ to the misreported label $\tilde{f}_i(x)=k$: $T_{j,k}=\PP(\tilde{f}_i(X)=k|\f_i(X) = j)$. Random flipping predictions will degrade the quality of a classifier. When there is no adversary attack, we focus on two kinds of misreport transition matrix: a uniform matrix or a sparse matrix. For the uniform matrix, we assume the probability of flipping from a given class into other classes to be the same: $T_{i,j}=T_{i,k}=e, \forall i\neq j\neq k$. $e$ changes gradually from 0 to 0.56 after 10 increases, which results in a 0\%--50\% misreport rate. The sparse matrix focuses on particular 5 pairs of classes which are easily mistaken between each pair. Denote the corresponding transition matrix elements of class pair $(i,j)$ to be: $(T_{ij}, T_{ji}), i\neq j$, we assume that $T_{ij}=T_{ji}=e, \forall (i,j)$. $e$ changes gradually from 0 to 0.5 after 10 increases, which results in a 0\%--50\% misreport rate.

%  For instance, an image of horse may look like a deer, while it is unlikely to be mistaken as a ship. Thus, we choose 5 pairs for either dataset

% \subsection{Uniform off-diagonal misreport model:}
% In certain real world scenarios, an agent refuses to truthfully report the prediction by randomly selecting another different class as the prediction. To simulate this case, we assume that 
  Every setting is simulated 5 times. The line in each figure consists of the median score of 5 runs as well as the corresponding ``deviation interval", which is the maximum absolute score deviation. The y axis symbolizes the averaged score of all test images.
  
  As shown in Figure~\ref{Fig:fig1_1},~\ref{Fig:fig2}, in most situations, 0-1 score and CE score of both $\A_{W}$ and $\A_{S}$ keep on decreasing while the misreport rate is increasing. As for 0-1 score without ground truth verification, the score of either agent begins to fluctuate more when the misreport rate in sparse misreport model is $>35\%$. %, which can be considered as a ``lower bound" for the s. \yl{what does lower bound mean?} \wjh{Once the misreport rate reaches such a ``lower bound", increasing the misreport rate does not have obvious impacts on the hypothesis score and is meaningless.} Thus, 
  Our results conclude that both the 0-1 score and CE score induce truthful reporting of a hypothesis and will penalize misreported agents whether there is ground truth for verification or not.

\begin{figure}[t]
%\makebox[7pt]{\raisebox{39pt}{\rotatebox[origin=c]{90}{\small{XXX}}}}%
    \centering
    \subfigure[\scriptsize 0-1 Score, Agent $A_{W}$]
{\includegraphics[width=.26\textwidth]{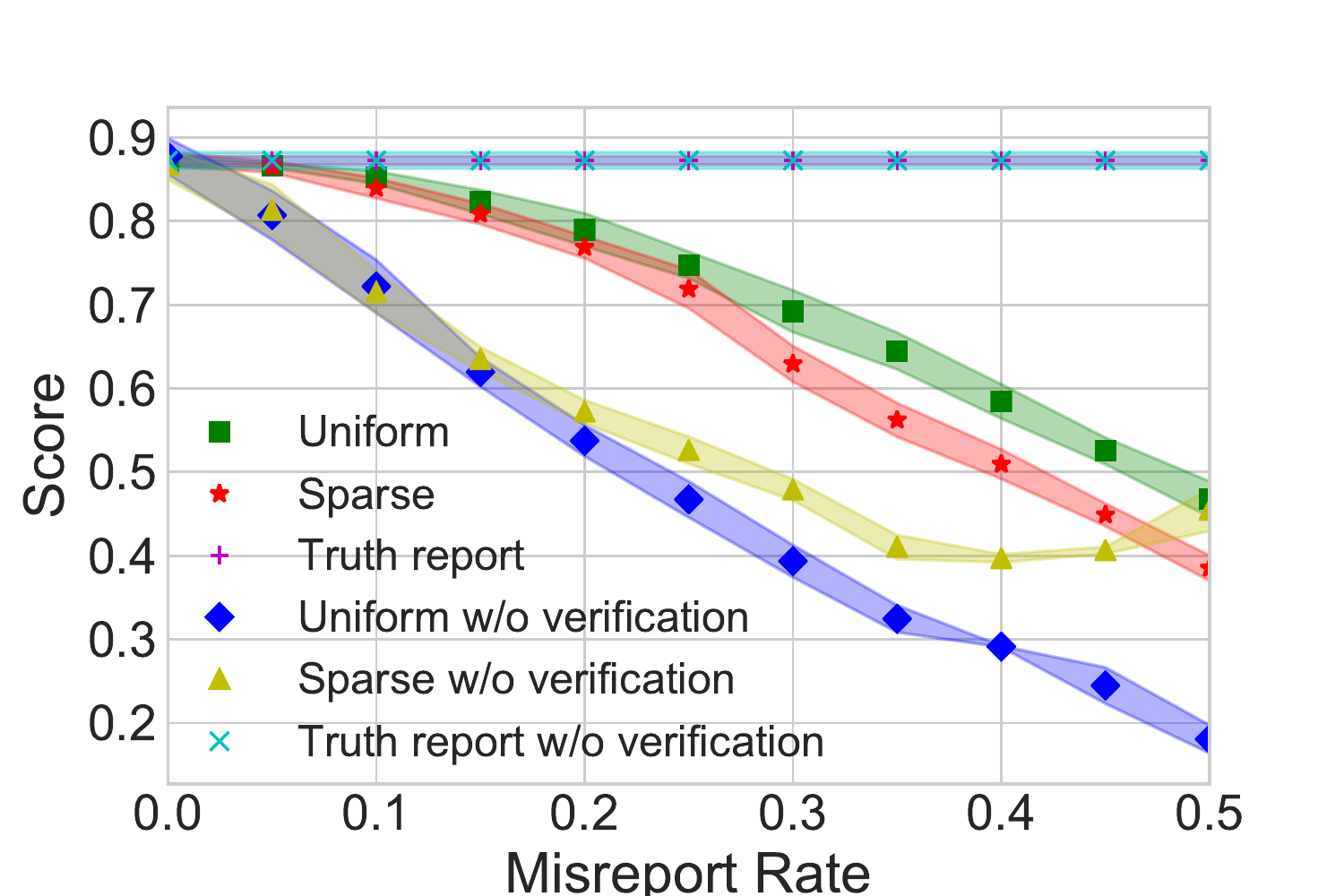}
\label{Fig: 1a}}
    \hspace{-0.28in}
    \subfigure[\scriptsize 0-1 Score, Agent $A_{S}$]
    {\includegraphics[width=.26\textwidth]{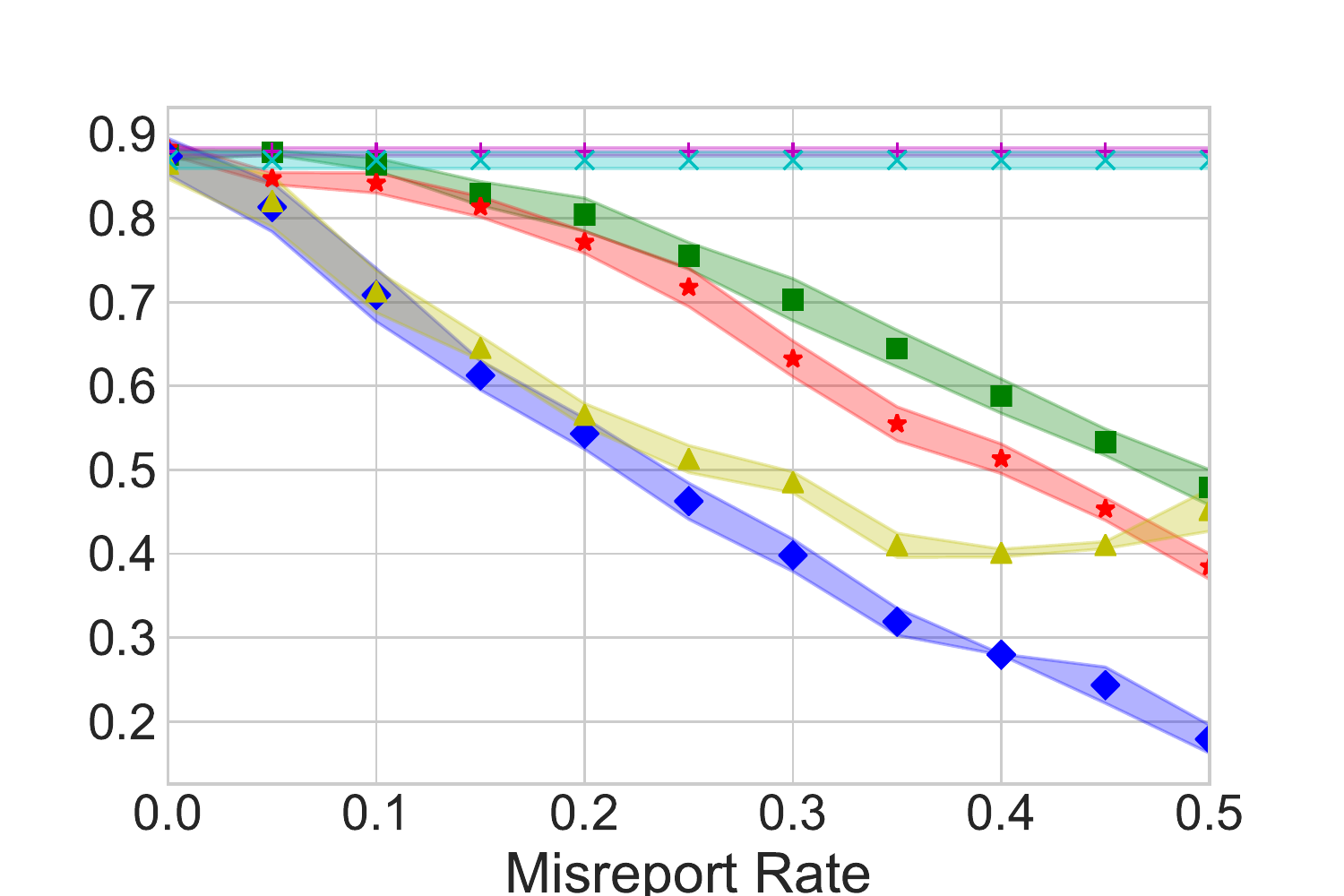}
    \label{Fig: 1b}
    }
    \hspace{-0.28in}
    \subfigure[\scriptsize CE Score, Agent $A_{W}$]
    {\includegraphics[width=.26\textwidth]{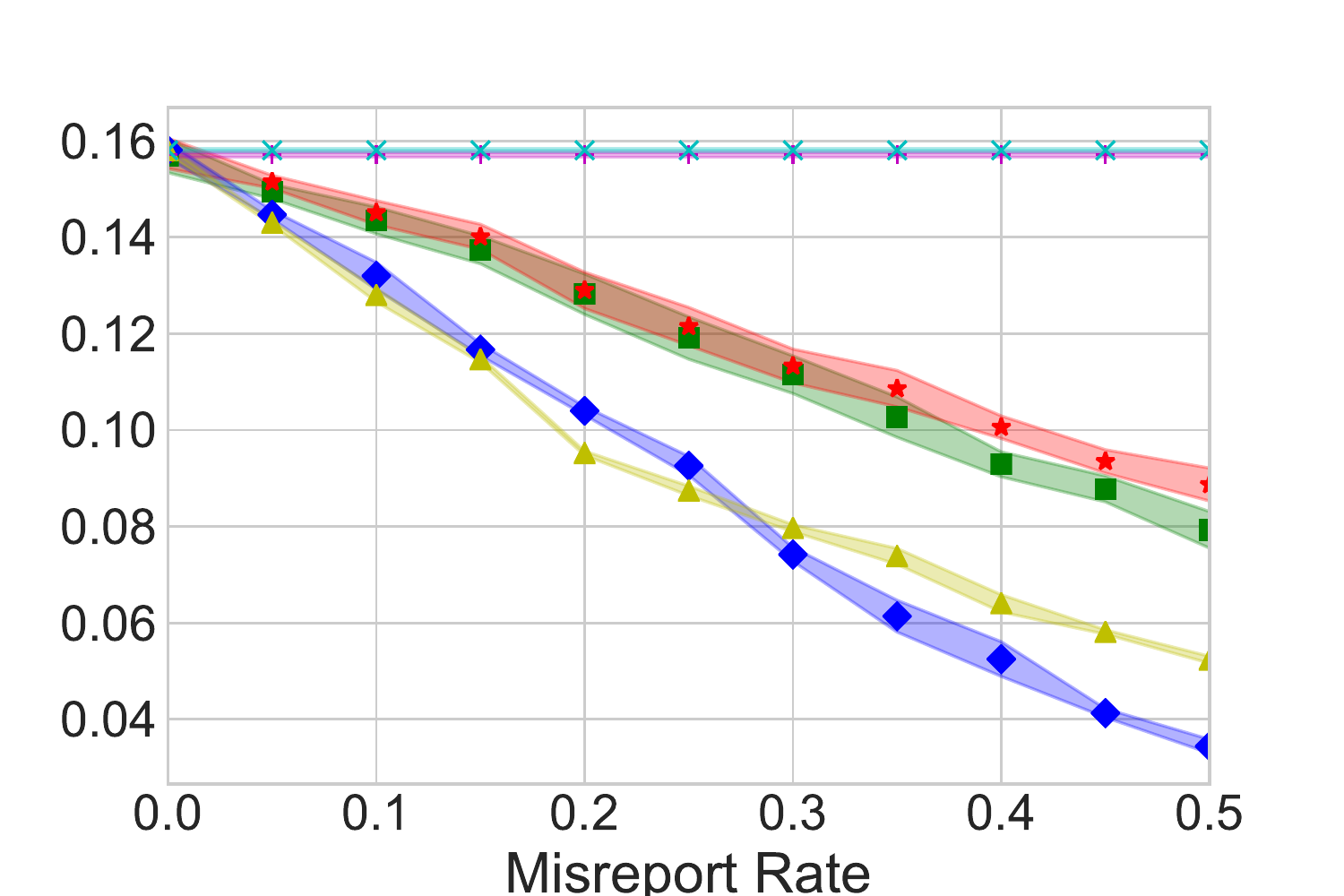}
    \label{Fig: 1c}}
    \hspace{-0.28in}
    \subfigure[\scriptsize CE Score, Agent $A_{S}$]
    {\includegraphics[width=.26\textwidth]{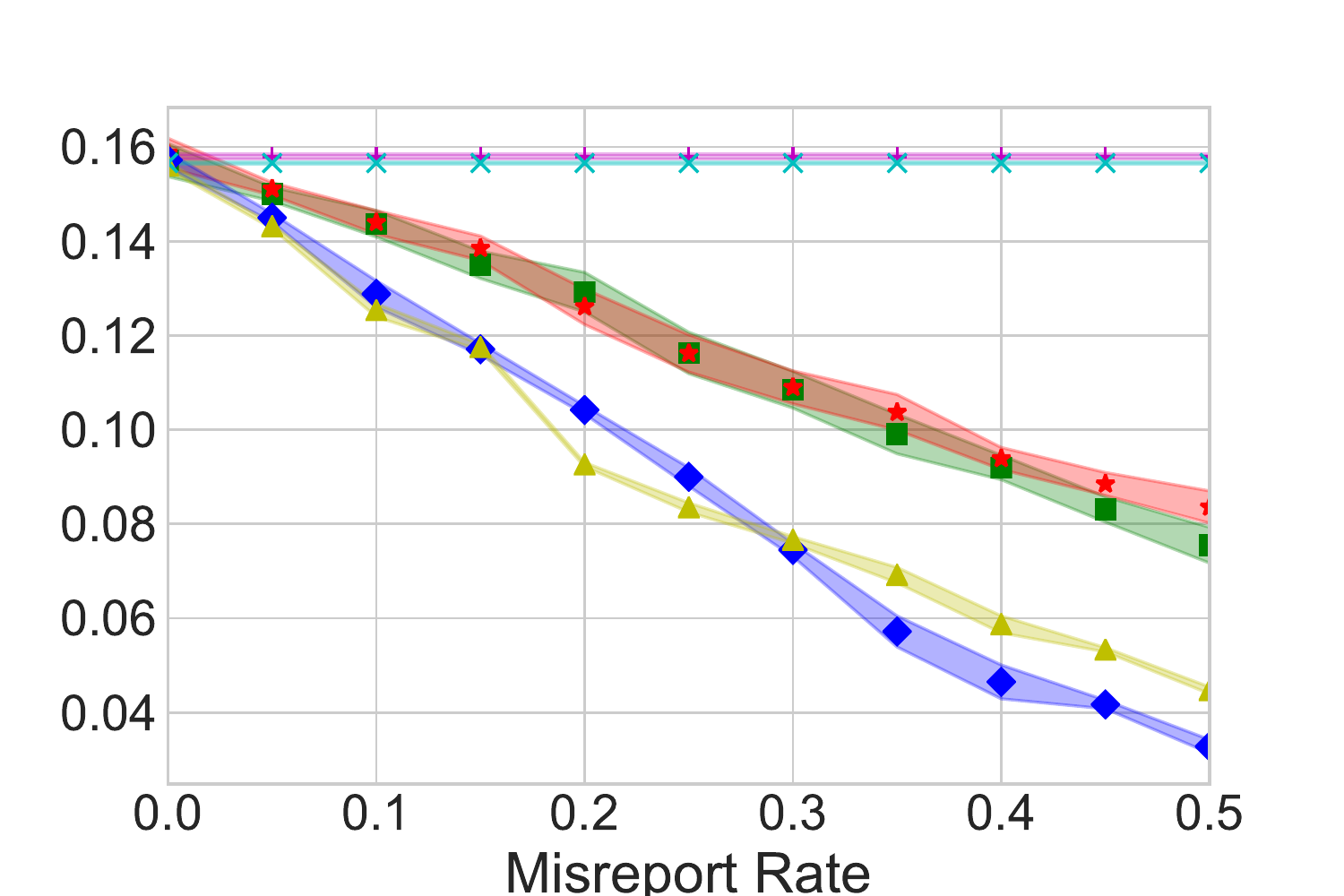}
    \label{Fig: 1d}}

        \vspace{-5pt}
        \caption{Hypothesis scores versus misreport rate on MNIST dataset.
        % \vspace{-18pt}
    }
    \label{Fig:fig1_1}
\end{figure}

\begin{figure}[t]
%\makebox[7pt]{\raisebox{39pt}{\rotatebox[origin=c]{90}{\small{XXX}}}}%
    \centering
    \subfigure[\scriptsize 0-1 score, agent $A_{W}$]
{\includegraphics[width=.26\textwidth]{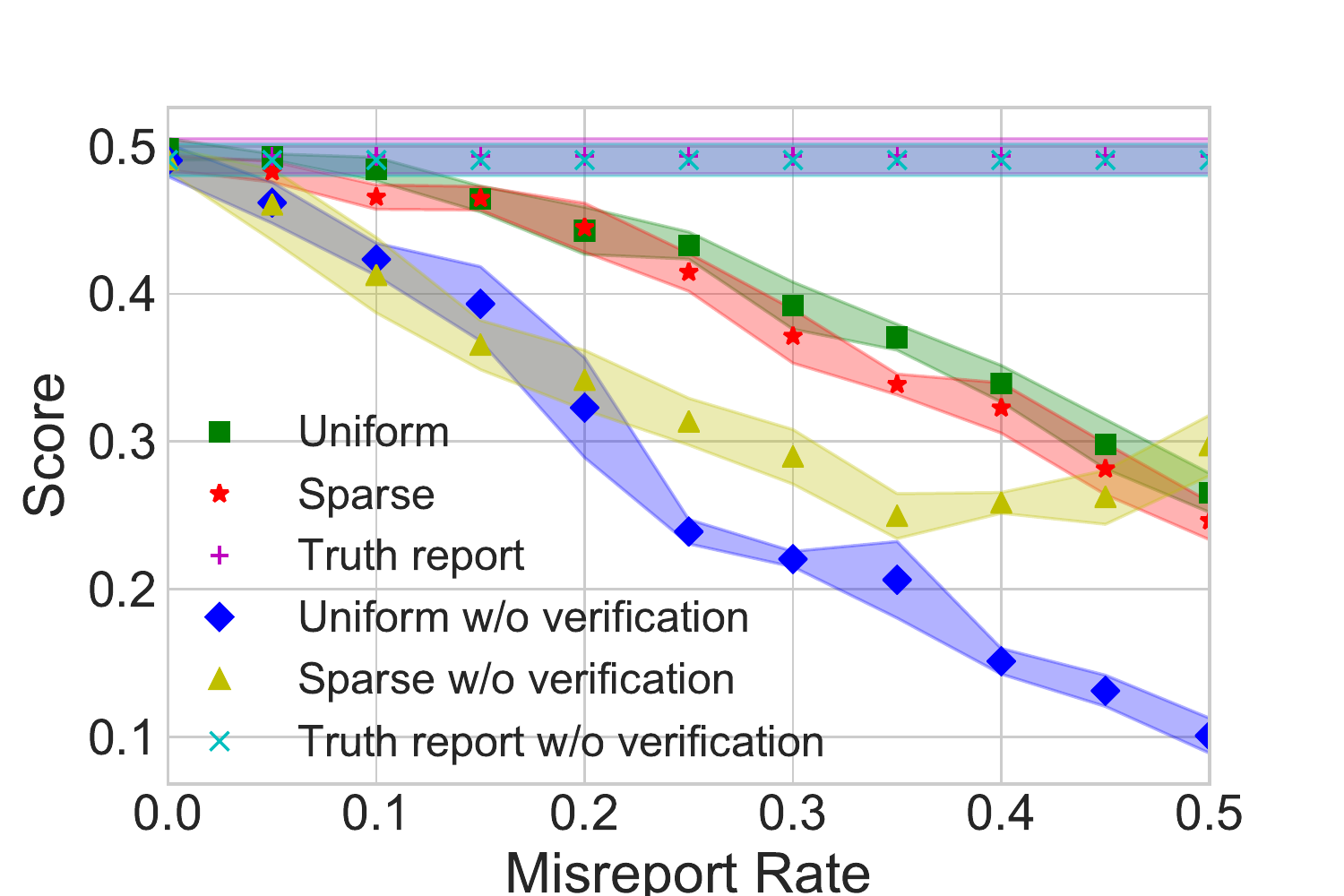}
\label{Fig: 2a}}
    \hspace{-0.28in}
    \subfigure[\scriptsize 0-1 score, agent $A_{S}$]
    {\includegraphics[width=.26\textwidth]{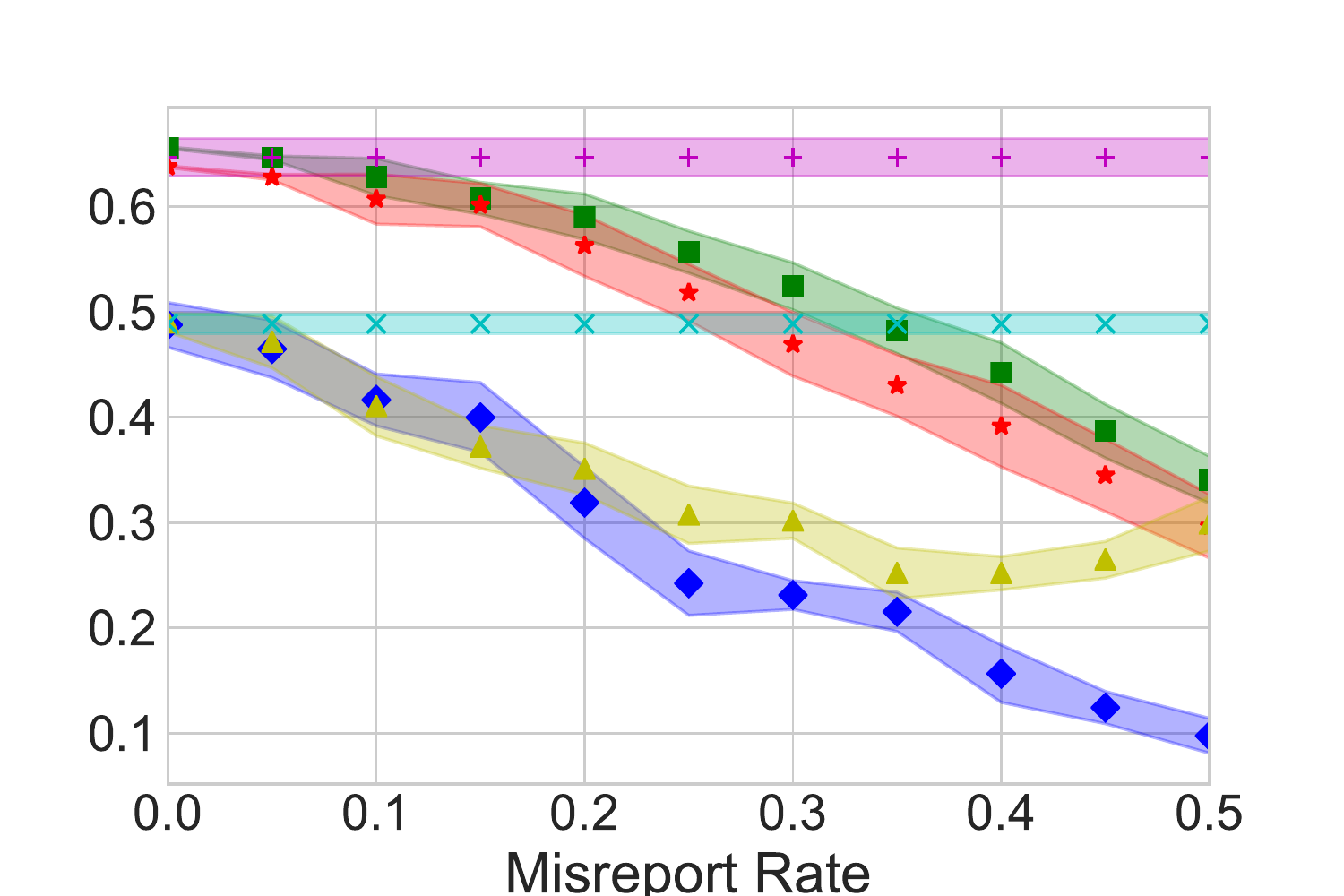}
    \label{Fig: 2b}}
    \hspace{-0.28in}
    \subfigure[\scriptsize CE score, agent $A_{W}$]
    {\includegraphics[width=.26\textwidth]{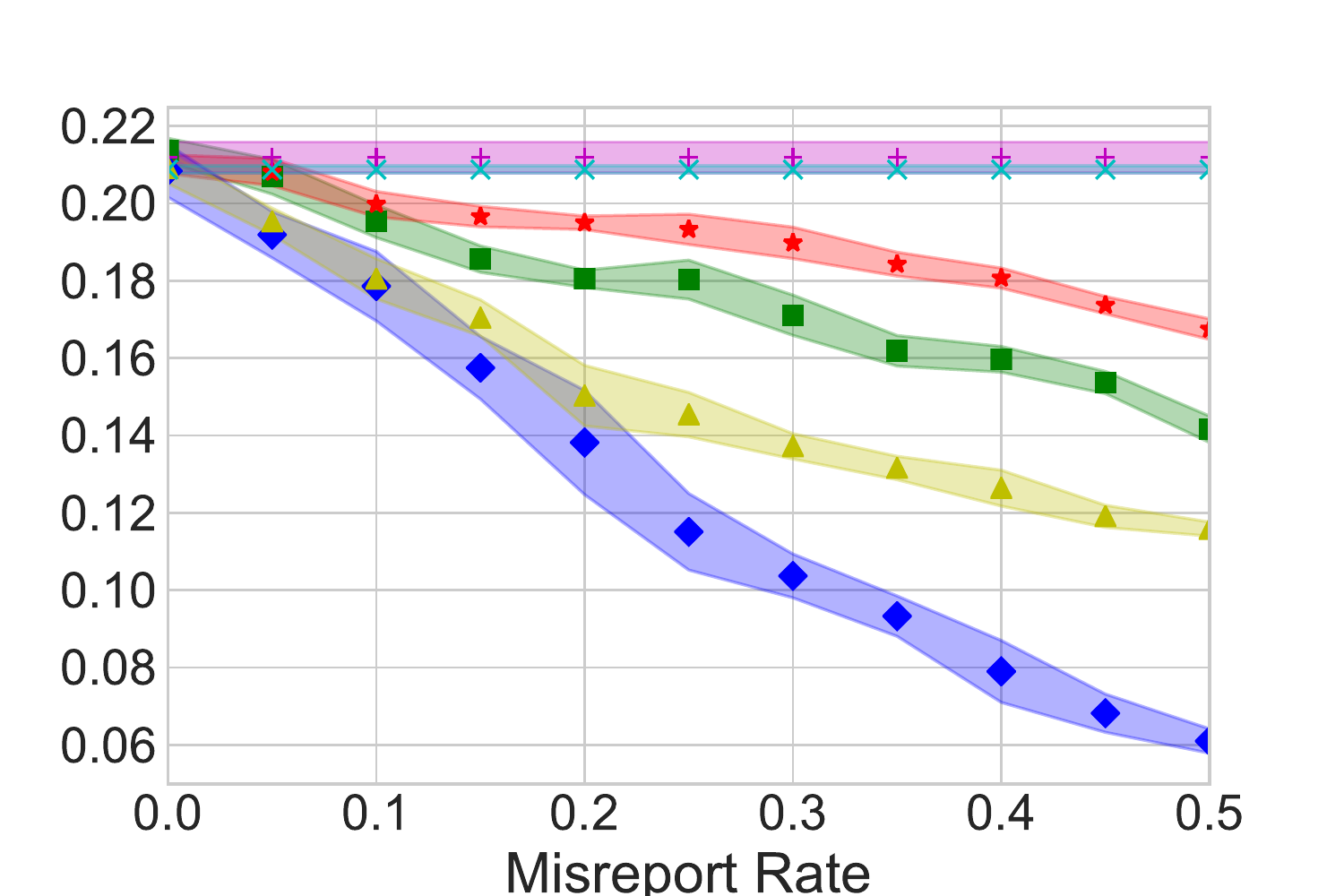}
    \label{Fig: 2c}}
    \hspace{-0.28in}
    \subfigure[\scriptsize CE score, agent $A_{S}$]
    {\includegraphics[width=.26\textwidth]{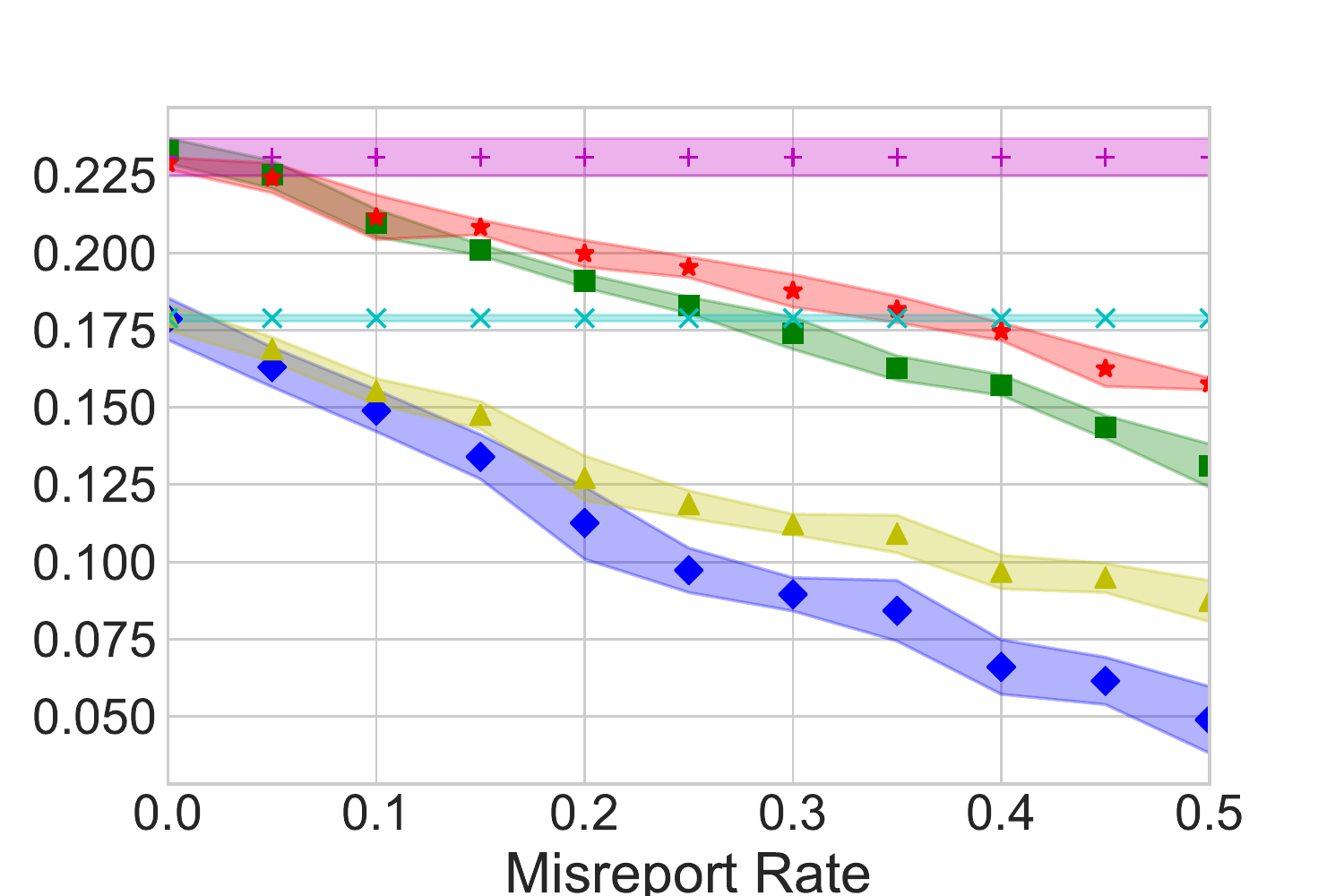}
    \label{Fig: 2d}}

        \vspace{-5pt}
        \caption{Hypothesis scores versus misreport rate on CIFAR-10 dataset. 
        % \vspace{-18pt}
    }
    \label{Fig:fig2}
\end{figure}

\subsection{Elicitation with adversarial attack}

We test the robustness of our mechanism when facing a 0.3-fraction of adversary in the participating population. We introduce an adversarial agent, LinfPGDAttack, introduced in AdverTorch~\cite{ding2019advertorch} to influence the labels for verification when there is no ground truth. In Figure~\ref{Fig:fig3}, both the 0-1 score and CE score induce truthful reporting of a hypothesis for MNIST. 

However, for CIFAR-10, with the increasing of misreport rate, the decreasing tendency fluctuates more often. Two factors attribute to this phenomenon: the agents' abilities as well as the quality of generated "ground truth" labels. When the misreport rate is large and generated labels are of low quality, the probability of successfully matching the misreported label to an incorrect generated label can be much higher than usual. But in general, these two scoring structures incentivize agents to truthfully report their results.

% for MNIST test dataset with respect to the uniform off-diagonal model, both $A_{W}$ and $A_{S}$ have a clear decreasing tendency no matter we choose 0-1 score or CE score. But for the sparse model, the deviation is a little bit large, see the yellow line and red line. For CIFAR-10 test dataset, as is shown in Figure~\ref{Fig: 3c},~\ref{Fig: 3d}, the decreasing tendency is less "clean" and seems to fluctuate more often. We 
% can attribute this phenomenon to 2 factors: one is that both agents are relative far from experts, the other one is that the generated "ground-truth" labels have pretty many wrong labels. As a result, when the misreport rate is super large and generated labels are of low quality, the probability of successfully matching the misreported label to an incorrect generated label can be much higher than usual. 

% When there is no ground-truth for verification, with the increasing number of agents, some of them may provide malicious predictions to disturb the stability of the environment. In order to simulate this case, w
% We calculate the weighted sum of the other agent's predicted probability (weights 0.7)  and the adversarial agent's predicted probability (weights 0.3) before the output prediction stage. The class with maximum probability is considered as the label for verification. 

% We repeat previous experiments but with no ground-truth for verification.

%\yl{for the adversarial set of experiments, I feel you }

\begin{figure}[ht]
%\makebox[7pt]{\raisebox{39pt}{\rotatebox[origin=c]{90}{\small{XXX}}}}%
    \centering
    \subfigure[\scriptsize 0-1 score, MNIST]
{\includegraphics[width=.26\textwidth]{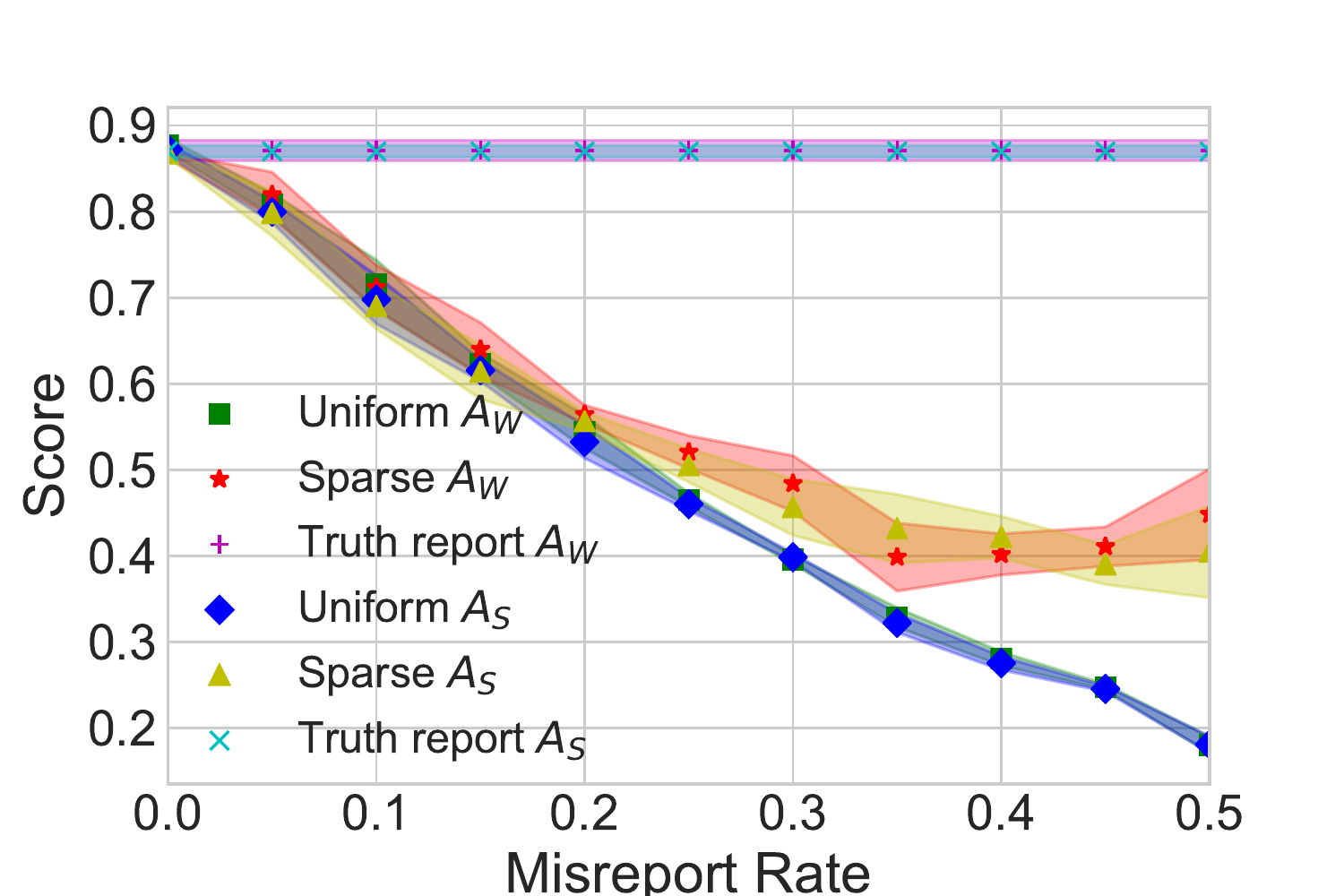}
\label{Fig: 3a}}
    \hspace{-0.28in}
    \subfigure[\scriptsize CE score, MNIST]
    {\includegraphics[width=.26\textwidth]{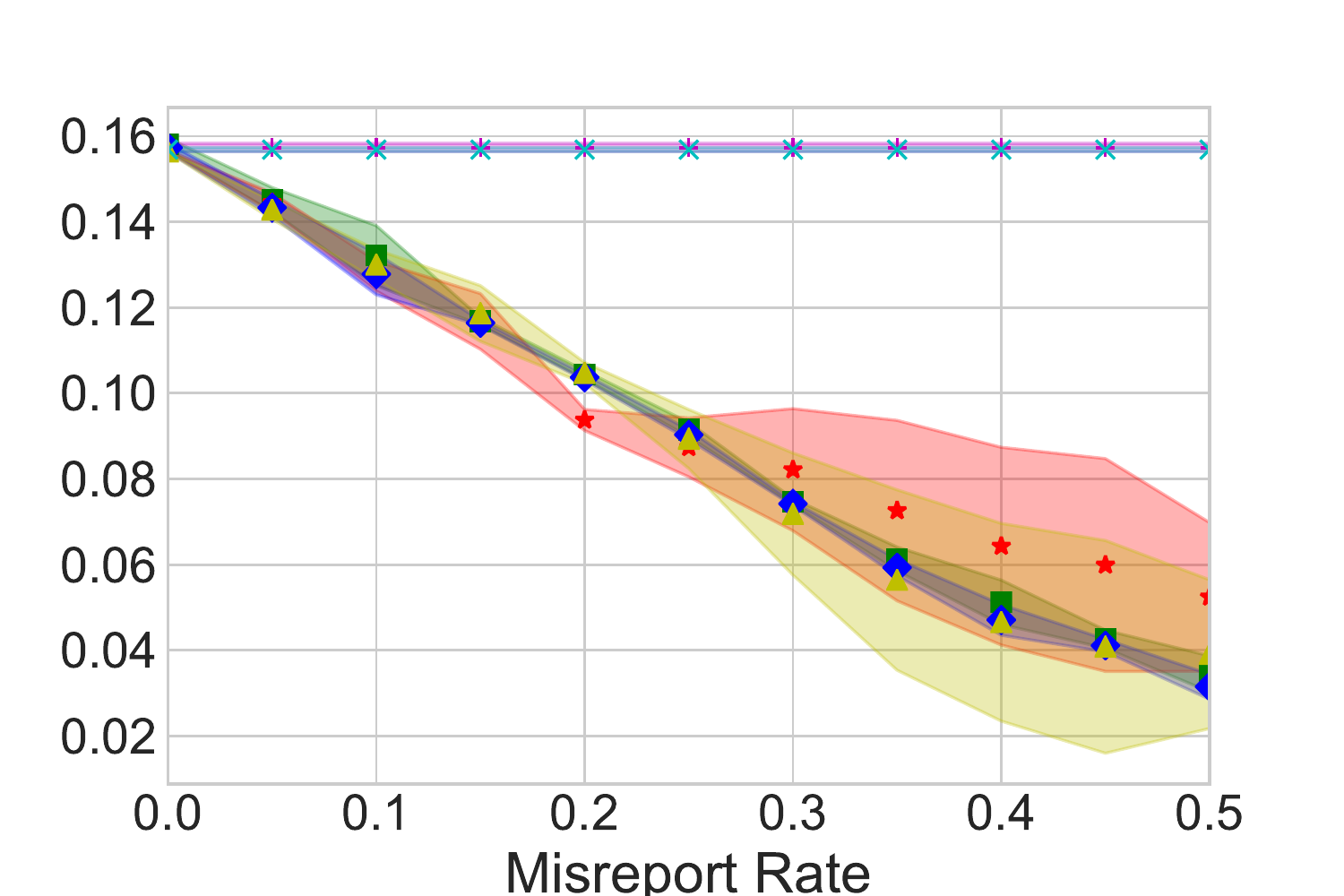}
    \label{Fig: 3b}}
    \hspace{-0.28in}
    \subfigure[\scriptsize 0-1 score, CIFAR]
    {\includegraphics[width=.26\textwidth]{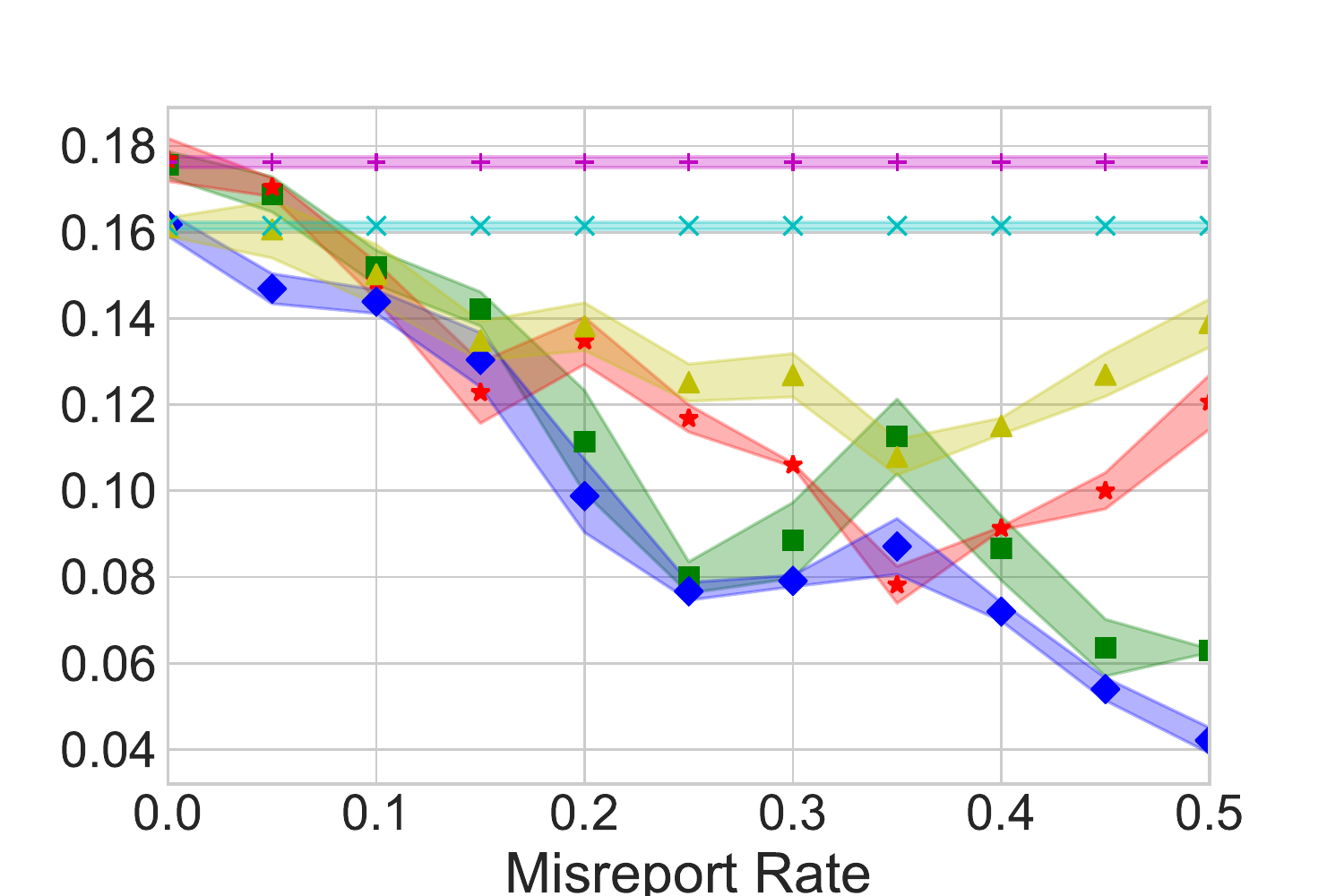}
    \label{Fig: 3c}}
    \hspace{-0.28in}
    \subfigure[\scriptsize CE score, CIFAR]
    {\includegraphics[width=.26\textwidth]{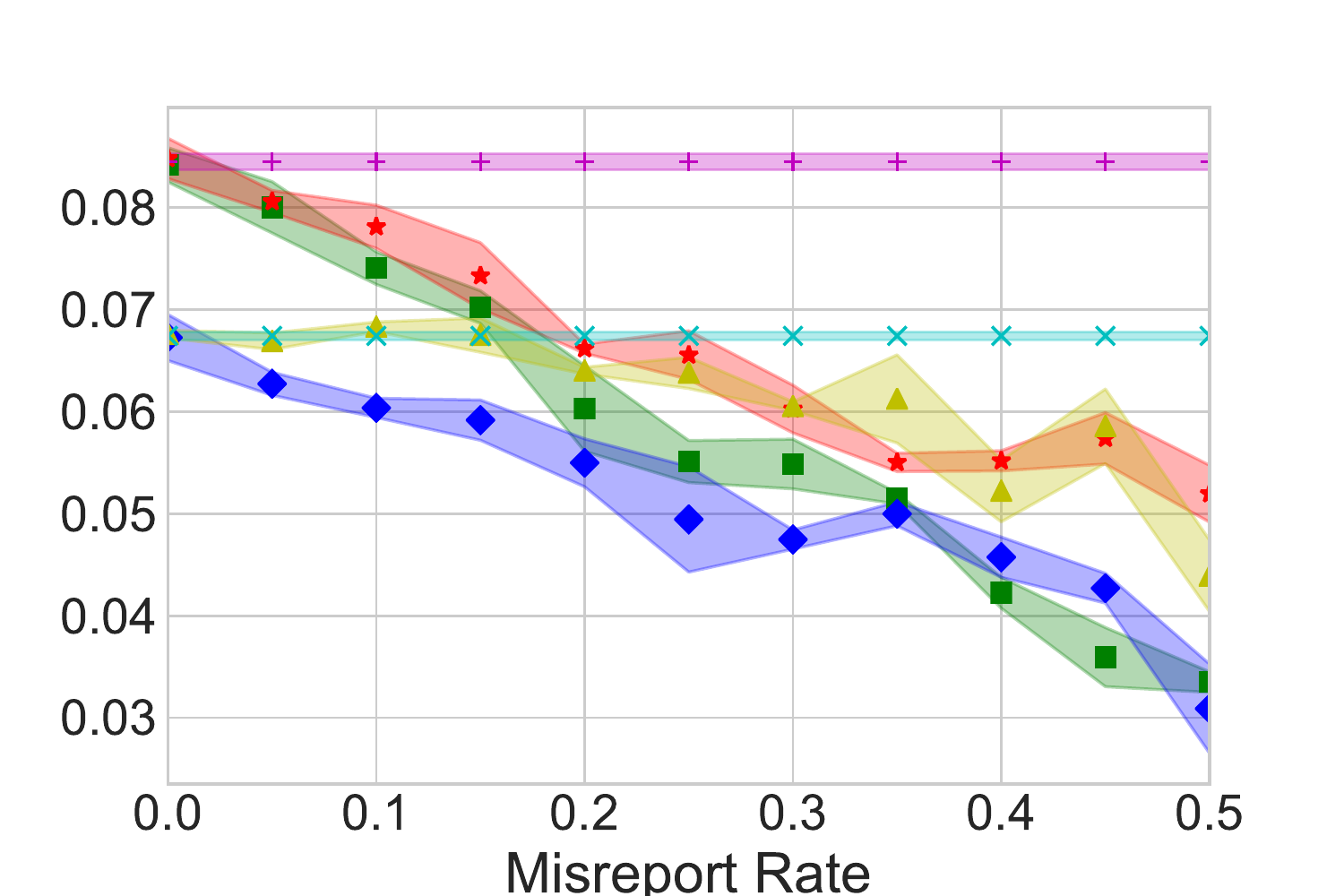}
    \label{Fig: 3d}}

        \vspace{-5pt}
        \caption{Hypothesis scores versus misreport rate (with adversarial attack). 
        % \vspace{-18pt}
    }
    \label{Fig:fig3}
\end{figure}

\section{Concluding remarks}
This paper provides an elicitation framework to incentivize contribution of truthful hypotheses in federated learning. We have offered a scoring rule based solution template which we name as hypothesis elicitation. We establish the incentive property of the proposed scoring mechanisms and have tested their performance with real-world datasets extensively. We have also looked into the accuracy, robustness of the scoring rules, as well as market approaches for implementing them.
\clearpage
\newpage
\bibliographystyle{plain}
\bibliography{library,myref,noise_learning, jiaheng}

\newpage
\section*{Appendix}
\subsection*{A. Proofs}
% \section*{Appendix}
\subsection*{Proof for Theorem \ref{thm:calibrate1}}

\begin{definition}\label{def:cal}  \cite{bartlett2006convexity}
A classification-calibrated loss function $\ell$ is defined as the follows: there exists a non-decreasing convex function $\Psi(\cdot)$ that satisfies:
$
\Psi\left(R(f) - R^*\right) \leq R_{\ell}(f) - R^*_{\ell}.
$ $\Phi(0) = 0$.
\end{definition}

\begin{proof}

Denote the Bayes risk of a classifier $f$ as 
$
R(f) := \mathbb{E}_{(X,Y)} \left[\mathbbm{1} \Big(f(X) \neq
    Y\Big) \right],
$ its minimum risk as $R^* = \argmin_{f} R(f)$. The classifier's $\ell$-risk is defined as 
$
R_{\ell}(f) := \mathbb{E}_{(X,Y)} \left[\ell(f(X),Y) \right]$, with its minimum value $R^*_{\ell} := \argmin_{f} R_{\ell}(f)$.

We prove by contradiction. Suppose that reporting a hypothesis $f$ returns higher payoff (or a smaller risk) than reporting $\f_i$ under $\ell$. Then 
\begin{align*}
    0 > R_{\ell}(f) - R_{\ell}(\f_i) \underbrace{=}_{(1)}  R_{\ell}(f)-R^*_{\ell}    \underbrace{\geq}_{(2)} \Psi\left(R(f) - R^*\right)\underbrace{\geq}_{(3)} \Psi(0) = 0,
\end{align*}
which is a contradiction. In above, the first equality (1) is by definition, (2) is due to the calibration condition, (3) is due to the definition of $R^*$.

\end{proof}

\subsection*{Proof for Lemma \ref{lemma:delta}}
% \wjh{In Lemma 1, we mention $f_{i}^*$ instead of $f$, in this proof, shall we change $f$ into $f_{i}^*$ to make the symbol keep consistency?}\yl{yes please revise}
\begin{proof}
The proof builds essentially on the law of total probability (note $f^*$ and $Y$ are the same):
\begin{align*}
    &\Delta^*(1,1) = \PP(f_{i}^*(X)=1,Y=1) - \PP(f_{i}^*(X)=1)\cdot \PP(Y=1)\\
    =&\PP(f_{i}^*(X)=1|Y=1) \cdot \PP(Y=1)- \PP(f_{i}^*(X)=1)\cdot \PP(Y=1)\\
    =&\PP(Y=1)\left(\PP(f_{i}^*(X)=1|Y=1) - \PP(f_{i}^*(X)=1)\right)\\
    =&\PP(Y=1)\biggl(\PP(f_{i}^*(X)=1|Y=1) - \PP(Y=1)\PP(f_{i}^*(X)=1|Y=1) \\
    &~~~~-\PP(Y=2)\PP(f_{i}^*(X)=1|Y=2)\biggr)\\
    =&\PP(Y=1)\bigl(\PP(Y=2) \cdot \PP(f_{i}^*(X)=1|Y=1)  - \PP(Y=2)\PP(f_{i}^*(X)=1|Y=2)\bigr)\\
    =&\PP(Y=1)\cdot \PP(Y=2) \cdot (\PP(f_{i}^*(X)=1|Y=1)-\PP(f_{i}^*(X)=1|Y=2))\\
    =&\PP(Y=1)\cdot \PP(Y=2) \cdot \left(1-\FNR(f_{i}^*) - \FPR(f_{i}^*)\right) > 0 
\end{align*}
Now consider $\Delta^*(1,2)$:
\begin{align*}
    &\Delta^*(1,2) = \PP(f_{i}^*(X)=1,Y=2) - \PP(f_{i}^*(X)=1)\cdot \PP(Y=2)\\
    =&\PP(f_{i}^*(X)=1|Y=2) \cdot \PP(Y=2)- \PP(f_{i}^*(X)=1)\cdot \PP(Y=2)\\
    =&\PP(Y=2)\bigl(\PP(f_{i}^*(X)=1|Y=2) - \PP(f_{i}^*(X)=1)\bigr)\\
    =&\PP(Y=2)\biggl(\PP(f_{i}^*(X)=1|Y=2) - \PP(Y=1)\PP(f_{i}^*(X)=1|Y=1) \\
    &~~~~- \PP(Y=2)\PP(f_{i}^*(X)=1|Y=2)\biggr)\\
    =&\PP(Y=2)\bigl(\PP(Y=1) \cdot \PP(f_{i}^*(X)=1|Y=2)  - \PP(Y=1)\PP(f_{i}^*(X)=1|Y=1)\bigr)\\
    =&\PP(Y=2)\cdot \PP(Y=1) \cdot (\PP(f_{i}^*(X)=1|Y=2)-\PP(f_{i}^*(X)=1|Y=1))\\
    =&\PP(Y=2)\cdot \PP(Y=2) \cdot \left(\FNR(f_{i}^*) + \FPR(f_{i}^*)-1\right) < 0 
\end{align*}
The second row of $\Delta^*$ involving $\Delta^*(2,1),\Delta^*(2,2)$ can be symmetrically argued. 
\end{proof}

%\begin{definition}\label{def:BI}
%$f$ is Bayesian informative about ground truth $Y$ if the following condition holds:
%\[
%\PP(Y=L|f(X)=L) > \PP(Y=L),~\forall L.
%\]
%\end{definition}

%\begin{proposition}For binary classification ($L=2$), 
%$\FNR(f) + \FPR(f) < 1$ is a sufficient and necessary condition that $f$ is Bayesian informative about $Y$. %, when $L=2$.
%\end{proposition}

\subsection*{Proof for Theorem \ref{thm:CA:truthful}}

\begin{proof}
Note the following fact:
$$
\E[S(f,\f)] := \sum_n \E[S(f(x_n),f^*(x_n))].
$$
Therefore we can focus on the expected score of individual sample $X$. The incentive compatibility will then hold for the sum. 

The below proof is a rework of the one presented in \cite{shnayder2016informed}:
% \wjh{In the 5th and 6th line of the equation, shall we replace $X$ with $X_{p1}$ and $X_{p2}$}
\begin{align*}
&\quad ~ \E\left[S(\rf_i(X),f^*(X))\right] \\
&= \E \left [ Sgn(\Delta^*(\rf_i(X),f^*(X)))-Sgn(\Delta^*(\rf_i(X_{p_1}),f^*(X_{p_2})))\right ]\\
&=\sum_{k \in [L]} \sum_{l \in [L]}\PP(\f_i(X) = k, f^*(X) = l) \\
&\quad\quad  \cdot \sum_{r \in [L]}\PP(\rf_i(X)=r|\f_i(X)=k) \cdot Sgn(\Delta^*(r,l))\\
&~~~~~-\sum_{k \in [L]} \sum_{l \in [L]}\PP(\f_i(X_{p_1}) = k)\cdot \PP( f^*(X_{p_2}) = l)\\
&\quad\quad \cdot \sum_{r \in [L]}\PP(\rf_i(X)=r |\f_i(X_{p_1})=k) \cdot Sgn(\Delta^*(r,l))\\
&=\sum_{k \in [L]} \sum_{l \in [L]}\PP(\rf_i(X) = k, f^*(X) = l) \\
&\quad\quad  \cdot \sum_{r \in [L]}\PP(\rf_i(X)=r|\f_i(X)=k) \cdot Sgn(\Delta^*(r,l))\\
&~~~~~-\sum_{k \in [L]} \sum_{l \in [L]}\PP(\f_i(X) = k)\cdot \PP( f^*(X) = l) \tag{replacing $X_{p}$ with $X$ due to iid assumption}\\
&\quad\quad \cdot \sum_{r \in [L]}\PP(\rf_i(X)=r |\f_i(X)=k) \cdot Sgn(\Delta^*(r,l))\\
&=\sum_{k \in [L]} \sum_{l \in [L]} \Delta^*(k,l) \cdot \sum_{r \in [L]} \PP(\rf_i(X)=r|\f_i(X)=k)\cdot Sgn(\Delta^*(r,l))
\end{align*}
Note that truthful reporting returns the following expected payment:
\begin{align*}
    &\sum_{k \in [L]} \sum_{l \in [L]} \Delta^*(k,l) \cdot \sum_{r \in [L]}\PP(\rf_i(X)=\f_i(X)=r|\f_i(X)=k) \cdot Sgn(\Delta^*(r,l))\\
    =&     \sum_{k \in [L]} \sum_{l \in [L]} \Delta^*(k,l) \cdot Sgn(\Delta^*(k,l)) \tag{Only the corresponding $k$ survives the 3rd summation}\\
    =& \sum_{k,l: \Delta^*(k,l) > 0} \Delta^*(k,l)
    %\\
    %\geq &\sum_{i \in [L]} \sum_{j \in [L]} \Delta(i,j) \cdot \sum_{r_1 \in [L]}\PP(r(X)=r_1|f(X)=i) M(r_1,j)
\end{align*}
Because $\sum_{r \in [L]}\PP(\rf_i(X)=r|\f_i(X)=k) \cdot Sgn(\Delta^*(r,l)) \in [0,1]$ we conclude that for any other reporting strategy:
\[
\sum_{k,l: \Delta^*(k,l) > 0} \Delta^*(k,l)\geq \sum_{k \in [L]} \sum_{l \in [L]} \Delta^*(k,l) \cdot \sum_{r \in [L]}\PP(\rf_i(X)=r|\f_i(X)=k) \cdot Sgn(\Delta^*(r,l))
\]
completing the proof.
\end{proof}

\subsection*{Proof for Theorem \ref{thm:accuracy}}

\begin{proof}
For any classifier $f$, the expected score is
\begin{align*}
    &\quad ~ \E[S(f(X),Y)] \\
    &= \PP(f(X) = Y) - \PP(f(X)=1)\PP(Y=1) - \PP(f(X)=2)\PP(Y=2) \quad~\tag{independence between $x_{p1}$ and $x_{p2}$}\\
    &= \PP(f(X) = Y) - \PP(f(X)=1)\cdot 0.5 - \PP(f(X)=2)\cdot 0.5 \quad~\tag{\text{equal prior}}\\
    &= \PP(f(X) = Y) - 0.5.
\end{align*}
The last equality indicates that higher than accuracy, the higher the expected score given to the agent, completing the proof.
\end{proof}

\subsection*{Calibrated CA Scores}

We start with extending the definition for calibration for CA:
\begin{definition}
We call $S_{\ell}$ w.r.t. original CA scoring function $S$ if the following condition holds:
\begin{align}
    &\Psi\left(\E[S(f(X),f^*(X))] - \E[S(\f_i(X),f^*(X))]\right)\nonumber \\
    &\leq \E[S_{\ell}(f(X),f^*(X))] - \E[S_{\ell}(\f_i(X),f^*(X))], \forall f. \label{calibration:S}
\end{align}
\end{definition}
 Since $S$ induces $\f_i$ as the maximizer, if $S_{\ell}$ satisfies the calibration property as defined in Definition \ref{def:cal}, we will establish similarly the incentive property of $S(\cdot)$
\begin{theorem}
If for a certain loss function $\ell$ such that $S_{\ell}$ satisfies the calibration condition, then $S_{\ell}$ induces truthful reporting of $\f_i$.
\end{theorem}
The proof of the above theorem repeats  the one for Theorem \ref{thm:calibrate1}, therefore we omit the details. 

Denote by $f^*_{\ell} = \argmin_f R_{\ell}(f)$. Sufficient condition for CA calibration was studied in \cite{liu2020peerloss}: 
\begin{theorem}[Theorem 6 of \cite{liu2020peerloss}]
Under the following conditions, $S_{\ell}$ is calibrated if $\PP(Y=1) = 0.5$, and $f^*_{\ell}$ satisfies the following:
$
\E[\ell(f^*_{\ell}(X),-Y)] \geq \E[\ell(f(X),-Y)],~\forall f.
$ %(2) $\alpha < 1, \max\{e_{+1},e_{-1}\} < 0.5$, and $\ell''(t,y) = \ell''(t,-y)$.
\end{theorem}
% \end{theorem}

\subsection*{Proof for Lemma \ref{col:delta}}
% \wjh{$\Delta (k,l)$ seems to be the same as $\Delta ^*(k, l)$ in the section before theorem 2}
\begin{proof}
Denote by
\begin{align*}
    &\Delta(k,l)  = \PP\bigl(\f_i(X)=k,f^*(X)=l\bigr)- \PP\bigl(\f_i(X) = k\bigr) \PP\bigl(f^*(X)= l\bigr), ~k,l \in \{1,2\}.
\end{align*}
i.e., the correlation matrix defined between $\f_i$ and ground truth label $Y$ ($f^*$). And
\begin{align*}
\Delta^*(k,l)  = \PP\bigl(\f_i(X)=k,\f_j(X)=l\bigr)- \PP\bigl(\f_i(X) = k\bigr) \PP\bigl(\f_j(X)= l\bigr)
\end{align*}
%Next we show that
%\[
%\Delta^*(k,l) = \left(1-\FNR(\f_j) - \FPR(\f_j)\right) \cdot \Delta(k,l) 
%\]
$\Delta^*$ further dervies
\begin{align*}
    \Delta^*(k,l)  &= \PP\bigl(\f_i(X)=k,\f_j(X)=l\bigr)- \PP\bigl(\f_i(X) = k\bigr) \PP\bigl(\f_j(X)= l\bigr)\\
    &=\PP\bigl(\f_i(X)=k,\f_j(X)=l|Y=1\bigr)\cdot \PP(Y=1) \\
    &~~~~+ \PP\bigl(\f_i(X)=k,\f_j(X)=l|Y=2\bigr)\cdot \PP(Y=2)\\
    &-\PP\bigl(\f_i(X) = k\bigr) \left(\PP\bigl(\f_j(X)= l|Y=1\bigr)\cdot \PP(Y=1) + \PP\bigl(\f_j(X)= l|Y=2\bigr)\cdot \PP(Y=2) \right)\\
    &=\PP(\f_i(X)=k|Y=1)\cdot \PP(\f_j(X)=l|Y=1)\cdot \PP(Y=1) \\
    &~~~~+ \PP(\f_i(X)=k|Y=1) \cdot \PP(\f_j(X)=l|Y=2)\cdot \PP(Y=2) \tag*{(by conditional independence)}\\
    &-\PP\bigl(\f_i(X) = k\bigr) \left(\PP\bigl(\f_j(X)= l|Y=1\bigr)\cdot \PP(Y=1) + \PP\bigl(\f_j(X)= l|Y=2\bigr)\cdot \PP(Y=2) \right)\\
    &=\PP(\f_j(X)=l|Y=1)\left( \PP(\f_i(X)=k|Y=1) \cdot \PP(Y=1)  - \PP\bigl(\f_i(X) = k\bigr)\cdot \PP(Y=1) \right) \\
    &+\PP(\f_j(X)=l|Y=2)\left( \PP(\f_i(X)=k|Y=2) \cdot \PP(Y=2)  - \PP\bigl(\f_i(X) = k\bigr)\cdot \PP(Y=2) \right)\\
    &=\PP(\f_j(X)=l|Y=1) \cdot \Delta(k,1) + \PP(\f_j(X)=l|Y=2) \cdot \Delta(k,2)\\
    &=\left(\PP(\f_j(X)=l|Y=1)-\PP(\f_j(X)=l|Y=2)\right) \cdot \Delta(k,1) ~~~~~~\tag*{because $(\Delta(k,1)+\Delta(k,2)=0)$}
\end{align*}
If $k=1,l=1$, the above becomes
\[
\left(1-\FNR(\f_j) - \FPR(\f_j)\right) \Delta(1,1) > 0
\]
If $k=1,l=2$, the above becomes
\[
\left(-1+\FNR(\f_j) + \FPR(\f_j)\right) \Delta(1,1) < 0
\]
If $k=2,l=1$, the above becomes
\[
\left(1-\FNR(\f_j) - \FPR(\f_j)\right) \Delta(2,1) < 0
\]
If $k=2,l=2$, the above becomes
\[
\left(-1+\FNR(\f_j) + \FPR(\f_j)\right) \Delta(2,1) > 0
\]
Proved.
\end{proof}

\subsection*{Proof for Theorem \ref{thm:pp:accuracy}}
% \wjh{Shall we delete the first equation after equation (5)? I think we can ignore the first equation and go directly to the 2nd equation}
\begin{proof}
%Let's short-hand $Y' = f'(X)$. The proof builds on a key property we prove 
%\[
%\E [S(f(X),Y')] = (1-\alpha' - \beta') \E[S(f(X),Y)]
%\]
%where 
%%$
%\alpha' := \PP(Y'=2|Y=1),~~\beta' := \PP(Y'=1|Y=2)
%$, as long as $Y'$ and $f(X)$ are conditionally independent given the ground truth label $Y$. The proof of this property is deferred to the appendix. 

Shorthand the following error rates:
\[
\alpha' = \PP(f'(X)=2|Y=1),~\beta' = \PP(f'(X)=1|Y=2) 
\]
When two classifiers $f, f'$ are conditionally independent given $Y$, next we establish the following fact:
\begin{align}
\E [S(f(X),f'(X))] = (1-\alpha' - \beta') \cdot \E[S(f(X),Y)].~\label{eqn:affine}
\end{align}
Short-hand $p:=\PP(Y=1)$. Then 
\begin{align*}
    &\E [S(f(X),f'(X))]\\
    =&\E [\BR(f(X),f'(X))]\\ &-\PP(f(X)=1)\PP(f'(X)=1)-\PP(f(X)=2)\PP(f'(X)=2)  \tag{second term in CA}\\
    =& p \cdot   \E [\BR(f(X),f'(X))|Y=1] + (1-p) \cdot   \E [\BR(f(X),f'(X))|Y=2]\\
    &-\PP(f(X)=1)\PP(f'(X)=1)-\PP(f(X)=2)\PP(f'(X)=2)  \\
    =& p \cdot   \E [\BR(f(X),1)] \cdot (1-\alpha') +p \cdot   \E [\BR(f(X),2)] \cdot \alpha' \tag{by conditional independence}\\
            &+ (1-p) \cdot   \E [\BR(f(X),1)]\cdot \beta' + (1-p) \cdot \E[\BR(f(X),2)]\cdot (1-\beta') \tag{by conditional independence}\\
    &-\PP(f(X)=1)\PP(f'(X)=1)-\PP(f(X)=2)\PP(f'(X)=2)  \\
    =& p \cdot (1-\alpha'-\beta') \cdot \E [\BR(f(X),1)] + (1-p) \cdot (1-\alpha'-\beta') \cdot \E [\BR(f(X),2)] \\
    &+ \alpha' \cdot \E [\BR(f(X),2)]  + \beta'\cdot \E [\BR(f(X),1)] \\
        &-\PP(f(X)=1)\PP(f'(X)=1)-\PP(f(X)=2)\PP(f'(X)=2)  \\
    %=& p \cdot (1-\alpha'-\beta') \cdot \E [\BR(f(X),1)]-p \cdot (1-\alpha'-\beta') \cdot \E[\BR(f(X),1)] \\
    %&+  (1-\alpha'-\beta') \cdot \E [\BR(f(X),2)] \\
    %&+ \alpha' \cdot \E [\BR(f(X),2)]  + \beta' \cdot \E [\BR(f(X),1)] \\
%    &-\PP(f(X)=1)\PP(f'(X)=1)-\PP(f(X)=2)\PP(f'(X)=2)\\
    =&(1-\alpha'-\beta') \cdot \E[\BR(f(X),Y)]\\
    &- \PP(f(X)=1) (\PP(f'(X)=1)-\beta') \\
    &- \PP(f(X)=2) (\PP(f'(X)=2)-\alpha') \tag{$*$}
    \end{align*}
Since
\begin{align*}
    \PP(f'(X)=1) &= \PP(Y=1) \cdot \PP(f'(X)=1|Y=1) + \PP(Y=2) \cdot \PP(f'(X)=1|Y=2) \\
    &= p \cdot (1-\alpha') + (1-p) \cdot \beta' 
\end{align*}
we have 
\[
\PP(f'(X)=1)-\beta' = p \cdot (1-\alpha'-\beta')
\]
Similarly we have
\[
\PP(f'(X)=2)-\alpha'=  (1-p) \cdot (1-\alpha'-\beta')
\]
Then
\begin{align*}
    (*)=&(1-\alpha'-\beta')  \cdot \E[\BR(f(X),Y)] -(1-\alpha'-\beta') \cdot\left(\PP(f(X)=1) \cdot p +\PP(f(X)=2)\cdot  (1-p)\right)\\
    =&(1-\alpha'-\beta') \cdot \left( \E[\BR(f(X),Y)] - (\PP(f(X)=1) \cdot  p +\PP(f(X)=2)\cdot  (1-p))\right)\\
    =&(1-\alpha'-\beta') \cdot \E[S(f(X),Y)]. \tag{definition of CA on $f(X)$ and $Y$}
\end{align*}
By condition (iii), replace $f,f'$ above with $\f_i,\f_j$ we know 
\[
\E[S(\f_i(X),\f_j(X))] = (1-\alpha_j - \beta_j)\E[S(\f_i(X),Y)]
\]
where
\[
\alpha_j = \PP(\f_j(X)=2|Y=1),~\beta_j = \PP(\f_j(X)=1|Y=2) 
\]
When condition (ii) considering an identify $\Delta^*$ we know that $\f_j$  is Bayesian informative, which states exactly that \cite{LC17}
\[
1-\alpha_j - \beta_j > 0
\]
\begin{definition}[Bayesian informativeness \cite{LC17}] \label{def:BI}
A classifier $f$ is Bayesian informative w.r.t. $Y$ if
\[
1 - \PP(f(X)=2|Y=1) - \PP(f(X)=1|Y=2) > 0.
\]
\end{definition}

Further we have proven that $\E[S(f(X),Y)]$ rewards accuracy under Condition (i):
\begin{align*}
    &\quad~\E[S(\f_i(X),Y)] = \PP(\f_i(X) = Y) - \PP(\f_i(X)=1)\PP(Y=1) - \PP(\f_i(X)=2)\PP(Y=2)\\
    &= \PP(\f_i(X) = Y) - \PP(\f_i(X)=1)\cdot 0.5 - \PP(\f_i(X)=2)\cdot 0.5\\
    &= \PP(\f_i(X) = Y) - 0.5,
\end{align*}
completing the proof.
\end{proof}

\subsection*{Proof for Theorem \ref{thm:market:pp1}}

\begin{proof}
The proof uses the fact that $\E [S(f(X),f'(X))] = (1-\alpha' - \beta') \cdot \E[S(f(X),Y)]$ proved in Theorem \ref{thm:pp:accuracy}, Eqn. (\ref{eqn:affine}):
\begin{align*}
    \E[S(\rf_{t}(x),f'(x)) - S(\rf_{t-1}(X),f'(x))
] = (1-\alpha' - \beta') (\E[S(\rf_{t}(x),Y) - S(\rf_{t-1}(X),Y)])
\end{align*}
Since $f'$ is Bayesian informative we know that $1-\alpha' - \beta'>0$ (Definition \ref{def:BI}). Using the incentive-compatibility property of $S$ when the market is closed with ground truth $Y$ (note that $S(\rf_{t-1}(X),Y)$ is independent from reporting $\rf_t$) we complete the proof.
\end{proof}

\subsection*{Peer Prediction Markets}

Our idea of making a robust extension is to pair the market with a separate ``survey" elicitation process that elicits redundant $C$ hypotheses:
\begin{algorithm}[H]
\caption{Market for Hypothesis Elicitation}\label{m:main}
\begin{algorithmic}[1]
\STATE Pay crowdsourcing survey participants using surveys and the CA mechanism;
\STATE Randomly draw a hypothesis from the surveys (from the $C$ collected ones) to close the market according to Eqn. (\ref{eqn:crowd:market}), up to a scaling factor $\lambda > 0$.
\end{algorithmic}
\end{algorithm}

Assuming the survey hypotheses are conditional independent w.r.t. the ground truth $Y$. Denote by $f'$ a randomly drawn hypothesis from the surveys, and
\[
\alpha' := \PP(f'(X) =2|Y=1),~~\beta' := \PP(f'(X)=1|Y=2)
\]
For an agent who participated both in survey and market will receive the following:
% \wjh{What is $f_{i(t)}$? Is it $f_t$ for agent $i$? It seems that in the proof, we just use $f_{t}$ rather than $f_{i(t)}$}
\[
S(\rf_i,f'_{-i}) + \lambda (S(\rf_{t},f') - S(\rf_{t-1},f'))
\]
Below we establish its incentive property:
\begin{theorem}\label{thm:crowd:market}
For any $\delta$ that
$
 \delta := \frac{\lambda}{ (1-\alpha'-\beta')\cdot (C+\lambda (C-1))} \underbrace{\longrightarrow}_{C \rightarrow \infty} 0 
$, agents have incentives to report a hypothesis that is at most $\delta$ less accurate than the truthful one.
\end{theorem}
\begin{proof}
The agent can choose to participate in either the crowdsourcing survey or the prediction market, or both. For agents who participated only in the survey, there is clearly no incentive to deviate. This is guaranteed by the incentive property of a proper peer prediction mechanism (in our case it is the CA mechanism that we use to guarantee the incentive property.). 

For agents who only participated in the market at time $t$, we have its expected score as follows:
\begin{align*}
&\E[S(\rf_t(X),f'(X))] - \E[S(\rf_{t-1}(X),f'(X))] \\
=& (1-\alpha' - \beta') \cdot \E[S(\rf_t(X),Y)] -(1-\alpha' - \beta') \cdot \E[S(\rf_{t-1}(X),Y)] \\
=& (1-\alpha' - \beta') (\E[S(\rf_t(X),Y)] - \E[S(\rf_{t-1}(X),Y)]),
\end{align*}
where in above we have used  Theorem 6, Eqn. (\ref{eqn:affine}). We then have proved the incentive compatibility for this scenario.

Now consider an agent who participated both on the survey and market. Suppose the index of this particular agent is $i$ in the survey pool and $t$ in the market. 

By truthfully reporting $\f_i,\f_t$ the agent is scored as 
\begin{align}
&\E[S(\f_i(X),\f_{-i}(X))] + \lambda (\E[S(\f_t(X),f'(X))] - \E[S(\f_{t-1}(X),f'(X))])\nonumber \\
=&\underbrace{\E[S(\f_i(X),\f_{-i}(X))]}_{\text{Term I}} + \lambda \cdot \underbrace{\frac{C-1}{C} \left(\E[S(\f_t(X),\f_{-i})]-\E[S(\f_{t-1}(X),\f_{-i})]\right)}_{\text{Term II: $1-1/C$ probability of seeing another survey }}\nonumber \\
&+\lambda \cdot \underbrace{\frac{1}{C}\cdot (\E[S(\f_t(X),\f_t(X))]-\E[S(\f_{t-1}(X),\f_t(X))])}_{\text{Term III: $1/C$ probability of seeing their own survey hypothesis for closing the market. }}\label{eqn:deviation}
\end{align}
We analyze each of the above terms:
\squishlist
    \item[\textbf{Term I}] Due to the incentive-compatibility of CA, the first term $\E[S(\f_i(X),\f_{-i}(X))]$ is maximized by truthful reporting. And the decrease in score with deviation is proportional to \footnote{Here we make a simplification that when the population of survey is large enough, $f_{-i}$ (average classifier by removing one agent) is roughly having the same error rates as $f'$.} 
    \[
 \E[S(f(X),\f_{-i})]=\lambda \cdot  (1-\alpha'-\beta')\cdot \E[S(f(X),Y)]
\]
\item[\textbf{Term II}] The second term $\frac{C-1}{C}(\cdot)$ is also maximized by truthful reporting, because it is proportional to $\E[S(\f_t(X),Y)]$ (application of Eqn. (\ref{eqn:affine}))
\[
\lambda \cdot \frac{C-1}{C}\cdot \E[S(\f_t(X),\f_{-i})]=\lambda \cdot \frac{C-1}{C} \cdot (1-\alpha'-\beta')\cdot \E[S(\f_t(X),Y)]
\]
\item[\textbf{Term III}] The third $1/C(\cdot)$ term is profitable via deviation. Suppose agent deviates to reporting some $\rf$ that differs from $\f_i$ by $\delta$ amount in accuracy:
\[
\delta := \E[\BR(\rf(X),Y) - \BR(\f_i(X),Y)]
\]
that is $\rf$ is $\delta$ less accurate than $\f_i$. Under CA, since $\E[S(f(X),Y)] =  \PP(f(X) = Y) - 0.5$ (from the proof for Theorem \ref{thm:accuracy}),
\begin{align*}
    &\E[S(\rf(X),\rf(X))]-\E[S(\f_{t-1}(X),\rf(X))]\\
    =&\E[\BR(\rf(X),\rf(X)) - \BR(\f_{t-1}(X),\rf(X))] \\
    =&1-\E[\BR(\f_{t-1}(X),\rf(X))] 
\end{align*}
%\E[S(\rf(X),\rf(X))]-\E[S(\f_{t-1}(X),\rf(X))]=\E[\BR(\rf(X),\rf(X)) - \BR(\f_i(X),\rf(X))]
%\]
%first it is easy to show that (due to the two indicator terms)
% \\
% \wjh{Since $\E[S(\rf(X),\rf(X))]-\E[S(\f_{t-1}(X),\rf(X))]=\E[\BR(\rf(X),\rf(X)) - \BR(\f_i(X),\rf(X))]$, the bound for the following inequality can be tighter? say, $[0,\dfrac{\lambda}{C}]$} \yl{you are right}
We have
\[
0 \leq \lambda \cdot \frac{1}{C}\cdot \left(\E[S(\rf(X),\rf(X))]-\E[S(\f_{t-1}(X),\rf(X))]\right) \leq \frac{\lambda}{C}
\]
Therefore the possible gain from the third term is bounded by $\frac{\lambda}{C}$.  %independent of agent $i$'s reporting
%\item the fourth term is roughly the matching probability; the potential gain by deviation is
%\[
%\lambda \cdot \frac{1}{K} \cdot 2 \Delta^*^2
%\]
%\item the last term's gain by deviation is 
%\[
%\lambda \cdot \frac{1}{K} \cdot (2\Delta^*^2+\Delta^*).
%\]
\squishend

On the other hand, again since $\E[S(f(X),Y)] =  \PP(f(X) = Y) - 0.5$ (from the proof for Theorem \ref{thm:accuracy}), the loss via deviation (the first two terms in Eqn. (\ref{eqn:deviation})) is lower bounded by
\begin{align*}
&\left(1+\lambda \cdot \frac{C-1}{C} \right)\cdot (1-\alpha'-\beta') \cdot \E[S(\rf(X),Y) - S(\f_i(X),Y)]\\
=&\left(1+\lambda \cdot \frac{C-1}{C} \right)\cdot (1-\alpha'-\beta') \cdot \E[\BR(\rf(X),Y) - \BR(\f_i(X),Y)]\\
=&\left(1+\lambda \cdot \frac{C-1}{C} \right)\cdot (1-\alpha'-\beta') \cdot \delta,
\end{align*}
where $1$ comes from the survey reward. Therefore when $C$ is sufficiently large such that
\[
\left(1+\lambda \cdot \frac{C-1}{C} \right)\cdot (1-\alpha'-\beta')\cdot \delta \geq  \frac{\lambda}{C}
 \]
i.e.,
 \[
 \delta \geq \frac{\lambda}{ (1-\alpha'-\beta')\cdot (C+\lambda (C-1))}
 \]
 Agent has no incentive to report such a $\rf$.
 
\end{proof}

\subsection*{Proof for Theorem \ref{thm:robust}}

\begin{proof}
Denote by $\tilde{\rf}(X)$ the reference classifier randomly drawn from the population. Then the expected score for reporting a hypothesis $f$ is given by:
\begin{align*}
    \E[S(f(X), \tilde{\rf}(X))] &= (1-\gamma) \cdot  \E[S(f(X), \f_{1-\gamma}(X))] +\gamma \cdot  \E[S(f(X), \f_{\gamma}(X))]\\
    &=(1-\gamma) \cdot (1-\alpha-\beta)\cdot \E[S(f(X), Y)] \\
    &~~~~+ \gamma \cdot (1-\hat{\alpha}-\hat{\beta}) \cdot  \E[S(f(X), Y)] %\cdot  \E[S(f(X), \f_{\gamma}(X))],
\end{align*}
where $\hat{\alpha},\hat{\beta}$ are the error rates of the adversarial classifier $\f_{\gamma}$:
\[
\hat{\alpha} := \PP(\f_{\gamma}(X)=2|Y=1),~\hat{\beta} := \PP(\f_{\gamma}(X)=1|Y=2).~
\] The second equality is due to the application of  Theorem 6, Eqn. (\ref{eqn:affine}) on $\f_{1-\gamma}(X)$ and $\f_{\gamma}(X)$.
%Applying Eqn. (\ref{eqn:affine}) again leads to%know that
%\[
% \E[S(f(X), \f_{\gamma}(X))] = (1-\hat{\alpha}-\hat{\beta}) \cdot  \E[S(f(X), Y)]
%\]

Therefore
\begin{align*}
  \E[S(f(X), \tilde{\rf}(X))] =\left((1-\gamma) \cdot (1-\alpha-\beta) + \gamma \cdot (1-\hat{\alpha}-\hat{\beta})\right) \cdot \E[S(f(X), Y)]
\end{align*}
Due to the incentive property of $\E[S(f(X), Y)]$, a sufficient and necessary condition to remain truthful is 
\[
(1-\gamma) \cdot (1-\alpha-\beta) + \gamma \cdot   (1-\hat{\alpha}-\hat{\beta})> 0
\]

Now we prove
\[
1-\hat{\alpha}-\hat{\beta} \geq - (1-\alpha^*-\beta^*)
\]

This is because the most adversarial classifier cannot be worse than reversing the Bayesian optimal classifier - otherwise if the error rate is higher than reversing the Bayes optimal classifier (with error rates $1-\alpha^*, 1-\beta^*$):
\[
\hat{\alpha} > 1-\alpha^*,~\hat{\beta} > 1-\beta^*
\]
we can reverse the adversarial classifier to obtain a classifier that performs better than the Bayes optimal one:
\[
1-\hat{\alpha} < \alpha^*,~1-\hat{\beta} < \beta^*
\]
which is a contradiction! Therefore
\[
\hat{\alpha} \leq 1-\alpha^*,~\hat{\beta} \leq 1-\beta^* \Rightarrow  1-\hat{\alpha}-\hat{\beta} \geq - (1-\alpha^*-\beta^*)
\]

Therefore a sufficient condition is given by 
\[
\frac{1-\gamma}{\gamma} > \frac{1-\alpha^*-\beta^*}{1-\alpha-\beta}
\]

\end{proof}

%\subsection*{Differentially private hypothesis}
%\paragraph{Output perturbation} A typical solution concept/framework to share hypothesis while preserving privacy of $D_i$ is via adding noise to $f_i$ (output perturbation):
%\[
%\tilde{\f}_i = \f_i + \nu(\epsilon,N)
%\]
%where $\nu(\epsilon,N)$ is a noise term drawn from a certain distribution, as a function of $\epsilon$ and $N$. The following theorem states the privacy guarantees of adding a properly chosen $\nu(\epsilon,N)$:
%\begin{theorem}
%When $\nu(\epsilon,N)$ is drawn XXXX, $\tilde{\f}_i$ preserves $\epsilon$-DP.
%\end{theorem}
%\paragraph{Output sampling} When the hypothesis space is more complicated and non-parametric, it is not easy to perturb the final hypothesis with an additive term. \cite{} proposes a sampling based method to preserve privacy.

\subsection*{B. Experiment details}
\subsection*{Training hyper-parameters}
We did not tune hyper-parameters for the training process, since we focus on hypothesis elicitation rather than improving the agent's ability/performance. We concentrate on different misreport transition matrices as well as the misreport rate which falls in the range of $[0.0, 0.5]$ during the hypothesis elicitation stage. In the original submitted version, the mentioned machine learning architecture for agent $\A_{S}$ and $\A_{W}$ is a typo. The architecture that we use in our experiments are specified below.
\begin{itemize}
  \item \textbf{MNIST}\\
  Agent $\A_{W}$ is trained on uniformly sampled 25000 training images from MNIST training dataset. The architecture is LeNet.  
  Agent $\A_{S}$ is trained on uniformly sampled 25000 training images (with a different random seed) from MNIST training dataset. The architecture is a 13-layer CNN architecture. Both agents are trained for 100 epochs. The optimizer is SGD with momentum 0.9 and weight decay 1e-4. The initial learning rate is 0.1 and times 0.1 every 20 epochs. 
  
  \item \textbf{CIFAR-10}\\
  Agent $\A_{W}$ is trained on uniformly sampled 25000 training images from CIFAR-10 training dataset. The architecture is ResNet34.  
  Agent $\A_{S}$ is trained on uniformly sampled 25000 training images (with a different random seed) from CIFAR-10 training dataset. The architecture is a 13-layer CNN architecture. Both agents are trained for 180 epochs. The optimizer is SGD with momentum 0.9 and weight decay 1e-4. The initial learning rate is 0.1 and times 0.1 every 40 epochs.
  
  \item \textbf{Adversarial attack}\\
   We  use  LinfPGDAttack,  introduced in AdverTorch~\cite{ding2019advertorch} to simulate the adversary untargeted attacks. We adopt an example parameter setting provided by AdverTorch: cross-entropy loss function, eps is 0.15, number of iteration is 40, maximum clip is 1.0 and minimum clip is 0.0.
   
\end{itemize}

\subsection*{Misreport models}
% As said in \ref{sec:exp_1}, each element of a misreport transition matrix $T_{j,k}=\PP(\tilde{f}_i(X)=k|\f_i(X) = j)$. 
\begin{itemize}
  \item \textbf{Uniform misreport model}\\
  In certain real world scenarios, an agent refuses to truthfully report the prediction by randomly selecting another different class as the prediction. We use the uniform misreport transition matrix to simulate this case. In our experiments, we assume that the probability of flipping from a given class into other classes to be the same: $T_{i,j}=T_{i,k}=e,  \forall i\neq j \neq k$. Mathematically, the misreport transition matrix can be expressed as:
    
{\tiny{\[
\begin{bmatrix}
    1-9e & e & e & e & e & e & e & e & e & e \\
    e & 1-9e & e & e & e & e & e & e & e & e\\
    e & e & 1-9e & e & e & e & e & e & e & e \\
    e & e & e & 1-9e & e & e & e & e & e & e\\
    e & e & e & e & 1-9e & e& e & e & e & e \\
    e & e & e & e & e & 1-9e& e & e & e & e \\
    e & e & e & e & e & e & 1-9e & e & e & e \\
    e & e & e & e & e & e & e& 1-9e & e & e \\
    e & e & e & e & e & e & e & e & 1-9e & e\\
    e & e & e & e & e & e & e & e & e & 1-9e
\end{bmatrix}\]}}
    We choose the value of $9e$ to be: $[0.0, 0.05, 0.10, 0.15, 0.20, 0.25, 0.30, 0.35, 0.40, 0.45, 0.50]$ for all mentioned  experiments in Section \ref{sec:exp}.
    
  \item \textbf{Sparse misreport model}\\
  For low resolution images or similar images, an agent is only able to make ambiguous decision. The agent may doubt whether his prediction could result in a higher score than choosing the other similar category or not. Thus, he may purposely report the other one which is not the original prediction. Even if an expert agent is able to distinguish confusing classes, it may still choose not to truthfully report the prediction since many other agents can not classify the image successfully. Without ground truth for verification, reporting truthfully and giving the correct prediction may result in a lower score than misreporting. For the sparse misreport transition matrix, we assume $T_{i,j}=T_{j,i}=e, \forall (i, j)$, $i\neq j$. Mathematically, the misreport transition matrix can be expressed as:
     
{\tiny{\[
\begin{bmatrix}
    1-e & 0 & e & 0 & 0 & 0 & 0 & 0 & 0 & 0 \\
    0 & 1-e & 0 & 0 & 0 & 0 & 0 & 0 & 0 & e \\
    e & 0 & 1-e & 0 & 0 & 0 & 0 & 0 & 0 & 0 \\
    0 & 0 & 0 & 1-e & 0 & e & 0 & 0 & 0 & 0 \\
    0 & 0 & 0 & 0 & 1-e & 0 & 0 & e & 0 & 0 \\
    0 & 0 & 0 & e & 0 & 1-e & 0 & 0 & 0 & 0 \\
    0 & 0 & 0 & 0 & 0 & 0 & 1-e & 0 & e & 0 \\
    0 & 0 & 0 & 0 & e & 0 & 0 & 1-e & 0 & 0 \\
    0 & 0 & 0 & 0 & 0 & 0 & e & 0 & 1-e & 0 \\
    0 & e & 0 & 0 & 0 & 0 & 0 & 0 & 0 & 1-e 
\end{bmatrix}\]}}
    We choose the value of $e$ to be: $[0.0, 0.05, 0.10, 0.15, 0.20, 0.25, 0.30, 0.35, 0.40, 0.45, 0.50]$ for all mentioned  experiments in Section \ref{sec:exp}.
  
\end{itemize}

\subsection*{Evaluation}
In the training stage, we use categorical cross-entropy as our loss function for evaluation. In the hypothesis elicitation stage, we choose two kinds of reward structure in the CA mechanism: 0-1 score and CE score. Let $f^*$ denote the ``optimal" agent, $f_i$ denote the agent waiting to be scored. Mathematically, 0-1 score (one-hot encode the model output probability list) can be expressed as:
$$
S(f_i(x_n),f^*(x_n)):=\BR\left(f_i(x_n) = f^*(x_n)\right)-\BR\left(f_i(x_{p_1})=f^*(x_{p_2})\right)
$$

CE score can be expressed as:
$$
S_{\ell_{CE}}(f_i(x_n),f^*(x_n)) = -\ell_{CE}(f_i(x_n),f^*(x_n)))-(-\ell_{CE}(f_i(x_{p_1}),f^*(x_{p_2}))).~\label{ca:calibrated}
$$

\begin{itemize}
  \item \textbf{Ground truth for verification}\\
  When there are ground truth labels for verification, $f^*(x)$ is equal to the corresponding ground truth label for a test image $x$.
  \item \textbf{No ground truth for verification}\\
  When there aren't ground truth labels for verification, we substitute the other agent's prediction for the ground truth label. Thus, $f^*(x):=\f_j(x), j\neq i$ for a test image $x$.
  \item \textbf{Adversarial attacks and no ground truth verification}\\
  To simulate the case when facing a 0.3-fraction of adversary in the participating population and without ground truth for verification, we use LinfPGDAttack to influence the labels for verification. Specifically, given a test image $x$, LinfPGDAttack will attack $x$ and generate a noisy image $x_{attack}$, we replace the ground truth labels with weighted agents' prediction. Mathematically, $f^*(x):=0.7\cdot \f_j(x)+0.3\cdot \f_j(x_{attack}), j\neq i$ for the test image $x$. The weighting procedure is implemented before one-hot encoding stage for $\f_j(x)$ and $\f_j(x_{attack})$. After that, the one-hot encoded label is considered as the ground truth label for verification. 
  
\end{itemize}

\subsection*{Statistics and central tendency}
\begin{itemize}
  \item \textbf{Misreport rate}\\
 The statistics ``misreport rate" in Figure~\ref{Fig:fig1_1},\ref{Fig:fig2} symbolizes the proportion of misreported labels. For example, in a uniform transition matrix, $T_{i,j}=T_{i,k}=e,  \forall i\neq j \neq k$, the misreport rate is $9e$. While in a sparse transition matrix setting, given $T_{i,j}=T_{j,i}=e, \forall (i, j)$, $i\neq j$, the misreport rate is exactly $e$.\\
 To see how adversarial attacks will affect the elicitation, in Figure~\ref{Fig:fig3}, we choose the same misreport transition matrices used in Figure~\ref{Fig:fig1_1},\ref{Fig:fig2} and calculate the misreport rate before applying the adversarial attacks. 
 
 \item \textbf{Central tendency}\\
 We run 5 times for each experiment setting as shown in Figure~\ref{Fig:fig1_1},\ref{Fig:fig2},\ref{Fig:fig3}. The central line is the mean of 5 runs. The ``deviation interval" (error bars) is the maximum absolute score deviation. For example, suppose in the elicitation with ground truth verification setting, we have 5 runs/scores for a uniform transition matrix with a 0.25 misreport rate: $[0.5, 0.5, 0.2, 0.1, 0.7]$, the mean is 0.4, the corresponding absolution score deviation is: $[0.1, 0.1, 0.2, 0.3, 0.3]$. Then the ``deviation interval" comes to: $0.4\pm 0.3$. Since the number of runs is not a large number, absolute deviation is no less than the standard deviation in our experiment setting.
 
\end{itemize}

\subsection*{Computing infrastructure}
In our experiments, we use a GPU cluster (8 TITAN V GPUs and 16 GeForce GTX 1080 GPUs) for training and evaluation.

\end{document}